\documentclass[11pt]{article}
\pdfoutput=1

\usepackage[margin=1in]{geometry}
\usepackage{setspace}
\usepackage{amsmath, amssymb, amsthm, bm, mathtools}
\usepackage{graphicx}
\usepackage{booktabs}
\usepackage{enumitem}
\usepackage{pgfplots}
\pgfplotsset{compat=1.17}
\usepackage{comment}
\usepackage{xcolor}
\usepackage[T1]{fontenc}    
\usepackage{hyperref}       
\usepackage{url}            
\usepackage{booktabs}       
\usepackage{amsfonts}       
\usepackage{nicefrac}       
\usepackage{microtype}      
\usepackage{xcolor}         
\usepackage{subcaption}
\usepackage{wrapfig}
\usepackage{booktabs}
\usepackage[normalem]{ulem}
\usepackage{comment}
\usepackage{enumitem}
\usepackage{float}

\newcommand{\argmax}{\operatorname*{argmax}}

\theoremstyle{plain}
\newtheorem{theorem}{Theorem}
\newtheorem{proposition}{Proposition}
\newtheorem{lemma}{Lemma}
\newtheorem{corollary}{Corollary}
\theoremstyle{definition}

\newtheorem{assumption}{Assumption}
\theoremstyle{remark}
\newtheorem{remark}{Remark}

\usepackage{bbm}
\usepackage{makecell}
\usepackage[authoryear]{natbib}

\bibpunct[, ]{(}{)}{,}{a}{}{,}

\title{Sampling Allocation of LinUCB: Optimal Design Limits in the Small-Gap Regime}

\author{%
  Yujie Liu \hspace{3em} Vincent Y.~F.~Tan \hspace{3em} Yunbei Xu\\
  National University of Singapore\\
  \texttt{\{yj\string-liu, vtan, yunbei\}@nus.edu.sg}
}
\date{}

\begin{document}

\maketitle

\begin{abstract}

We study the sampling allocation of LinUCB in the small-gap regime, where the reward gaps are of order at most $n^{-1/2}$ over the decision horizon $n$.  This scaling captures the hard instances underlying worst-case regret lower bounds, for which LinUCB is known to be near optimal up to logarithmic factors in $n$. Using a mean-field perspective, we characterize this allocation through the empirical sampling distribution, a macroscopic object that averages the effect of adaptive decisions over the horizon, and identify its limit as~$n\to\infty$.  We establish that in this regime, the empirical sampling distribution induced by LinUCB converges to the set of D-optimal designs. This central result reveals that, in the small-gap regime, LinUCB not only achieves near optimal minimax regret but also allocates samples in a way that is asymptotically efficient for learning the reward parameter, thereby connecting regret-driven online learning with information-efficient experimental design. Building on the optimal design limit, we obtain two useful consequences. First, we refine the asymptotic regret analysis of LinUCB in the small-gap regime by characterizing its leading-order constant in the limit. Second, we show that, despite LinUCB's adaptive sampling strategy, the regularized least-squares estimator satisfies a central-limit-type theorem in the small-gap regime, thereby enabling valid statistical inference for the reward parameter.

\end{abstract}

\section{Introduction}
Linear bandits provide a canonical framework for sequential decision making with correlated rewards and have become a fundamental model for optimizing machine learning systems under uncertainty~\citep{LiEtAl2010}. In such applications, a learner repeatedly chooses among competing actions, such as personalized recommendation~\citep{LiEtAl2010}, cloud configurations~\citep{AlipourfardEtAl2017}, or adaptive video streaming~\citep{MaoEtAl2017}, and observes noisy performance feedback from real system operations. By learning from this feedback, the learner aims to improve system performance over time and make increasingly informed decisions.

Traditional analyses of linear bandits have primarily focused on regret minimization~\citep{DaniEtAl2008,ChuEtAl2011,abbasi2011improved}. Regret
is a cumulative reward criterion: it evaluates the total performance loss
relative to an oracle policy.  While regret is the fundamental benchmark
for decision quality, it does not directly describe the algorithm's
{\it sampling allocation}, namely how samples are adaptively allocated across the action space over time. This perspective dates back to the seminal work of~\citet{LaiRobbins1985}, whose terminology---``adaptive allocation rules''---reflects the problem's origins in experimental design. 
We study this complementary allocation question for linear bandits. The sampling
allocation is a finer object than cumulative regret: it records how
information accumulates, explains
the mechanism behind regret guarantees, and provide the right foundation
for refined asymptotic questions, including asymptotic regret and  statistical inference for the reward
parameter. Motivated by this perspective, we focus on the sampling allocation of LinUCB, one of the most widely used algorithms for linear bandits in real-world applications. LinUCB achieves
near optimal minimax regret guarantees up to logarithmic factors~\citep{ChuEtAl2011}, yet its  sampling behavior remains much
less understood.

We characterize LinUCB's sampling allocation in the {\it small-gap regime}, which captures a worst-case setting for regret minimization. This regime is formalized through a triangular-array framework: in the $n$th instance, the reward gaps are of order at most $n^{-1/2}$ over the decision horizon $n$. This $n^{-1/2}$ scaling underlies the hard instances used to establish minimax regret lower bounds for linear bandits~\citep{dani2007price,lattimore2020bandit}. Since LinUCB achieves near optimal minimax regret guarantees up to logarithmic factors, this regime provides a natural setting for studying the sampling allocation underlying its worst-case regret performance. Intuitively, gaps on the $n^{-1/2}$ scale keep differences in expected rewards comparable to statistical estimation errors over the horizon, so identifying the optimal action remains challenging. In practice, this captures settings such as online experiments comparing product variants, where the available experimental budget is insufficient to reliably distinguish between alternatives, yet repeated suboptimal decisions can still lead to non-negligible cumulative losses.


In this paper, we ask the following central question:
\begin{quote}{\it How can we characterize the sampling allocation of LinUCB in the small-gap regime?}
\end{quote}

Characterizing this sampling allocation of an adaptive linear bandit algorithm is challenging because its decisions are history-dependent, feedback-driven, and coupled through a correlated linear reward model. To make this allocation behavior analytically tractable, we adopt a \emph{mean-field} perspective in the classical statistical mechanics sense: rather than tracking the microscopic sequence of adaptive decisions, we characterize the limiting behavior of a macroscopic object, namely the empirical sampling distribution, as the decision horizon grows. This object averages the effect of adaptive decisions over the horizon.

Under this mean-field perspective, we answer the above question by identifying the limit of the empirical sampling distribution in the small-gap regime. Specifically, we show that the empirical sampling distribution induced by LinUCB converges to the set of D-optimal designs~\citep{Kiefer1959}. Thus, LinUCB’s long-run sampling allocation is governed by an optimal design principle, revealing that in the small-gap regime, LinUCB not only
achieves near optimal minimax regret but also allocates samples in a way that is asymptotically efficient for
learning the reward parameter.


Together with this result, the mean-field perspective manifests itself in three related ways. First, in the macroscopic sense, the empirical sampling distribution summarizes the average effect of adaptive decisions. Inspired by statistical mechanics~\citep{Ellis1985Entropy}, we view individual decisions as microscopic dynamics and the empirical sampling distribution as a macroscopic object that records their collective evolution. Second, the mean-field description becomes asymptotically exact in the large-horizon, small-gap regime.
 Third, under LinUCB, the optimal-design limit shows that this macroscopic limit is characterized by the maximizers of a potential function, namely the D-optimality criterion~\citep{Kiefer1959}, a structure that does not arise in the same form for ordinary multi-armed bandits~\citep{KalvitZeevi2021,KhamaruZhang2024InferenceUCB,Han2026ThompsonSampling}. These three aspects—the macroscopic interpretation, the limiting framework, and the potential function characterization—are summarized in Table~\ref{tab:meanfield_linucb}.

\begin{table}[htbp]
\centering
\small
\setlength{\tabcolsep}{10pt}
\renewcommand{\arraystretch}{1.5}
\begin{tabular}{p{2cm} p{6cm} p{6cm}}
\toprule
 & \textbf{Mean-field statistical mechanics} & \textbf{LinUCB in small-gap regime} \\
\midrule

\textbf{Mean-field viewpoint}
&
Average many interacting microscopic degrees of freedom into a macroscopic object
&
Average adaptive sampling decisions over the decision horizon into the empirical sampling distribution
\\[0.5em]
\midrule
\textbf{Microscopic variables}
&
Microscopic component states, e.g., $(X_1,\ldots,X_N)$
&
Adaptive sampling decisions
$(A_1^{(n)},\ldots,A_n^{(n)})$
\\[0.5em]

\textbf{Macroscopic object}
&
An order parameter or empirical observable, e.g., $\mu_N=\frac1N\sum_{i=1}^N f({X_i})$

&
The empirical sampling distribution
$
\pi_n^{(n)}=\frac1n\sum_{t=1}^n \delta_{A_t^{(n)}}
$
\\
\midrule
\textbf{System size}
&
Number of components $(N\to\infty)$
&
Decision horizon $(n\to\infty)$
\\

\textbf{Fixed scaling quantity}
&
Keep suitable intensive parameters fixed, such as density or temperature
&
Small-gap regime: reward gaps $\Delta^{(n)}$ scale \makecell{with the horizon $n$ according to}
{$
\Delta^{(n)}=O( n^{-1/2})
$}
\\
\midrule
\textbf{Potential function}
&
A variational objective characterizing the macroscopic state
&
The D-optimality criterion
$\Phi(\pi)=\log\det M(\pi)$, where $M(\cdot)$ is given in~\eqref{def:info_matrix}
\\[1em]

\textbf{Limit characterization}
&
The limiting macroscopic state is characterized by optimizers of the potential function
&
\makecell{The empirical sampling distribution con-}{verges to
$\mathcal P_D=\argmax_{\pi\in\mathcal P(\mathcal A)}
\Phi(\pi)$} (Theorem~\ref{thm:linucb-general-D-optimal})
\\[1em]
\bottomrule
\end{tabular}
\caption{A Mean-Field Perspective on LinUCB in Small-Gap Regime}
\label{tab:meanfield_linucb}
\end{table}

\subsection{Contributions}

We summarize our contributions below.

\begin{itemize}

\item \textbf{An optimal-design limit for LinUCB in small-gap regime.}
We establish that in the small-gap regime, the empirical sampling distribution generated by LinUCB converges to the set of D-optimal designs. In our setting, D- and G-optimality are equivalent by the equivalence theorem of~\citet{Kiefer_Wolfowitz_1960}. Both are classical optimal-design criteria for efficient parameter estimation~\citep{AtkinsonDonevTobias2007}.
Our result therefore shows that, in the small-gap regime, although LinUCB is designed for regret minimization, its long-run sampling allocation is also asymptotically efficient for parameter learning in the D-optimality sense.

The proof of this result further reveals the mechanism behind the optimal-design limit. As the horizon $n$ grows, reward differences become small relative to the exploration bonus. Consequently, after a negligible initial period, LinUCB behaves like an approximate variance-maximization procedure, repeatedly selecting actions with nearly maximal uncertainty in the estimated reward. This is closely related to the classical Fedorov--Wynn procedure for constructing D-optimal designs, which sequentially selects design points by maximizing the current variance function~\citep{Wynn1970,Fedorov1972}. 
We extend this convergence mechanism to an approximate triangular-array setting and show that it drives the empirical sampling distribution toward the set of D-optimal designs.

\item \textbf{A mean-field perspective on sampling allocation in small-gap regime.}
As summarized in Table~\ref{tab:meanfield_linucb}, we develop a mean-field perspective in which the individual adaptive decisions of LinUCB are viewed as microscopic dynamics, while the empirical sampling distribution serves as a macroscopic state that averages their aggregate behavior over the decision horizon.
The role of this perspective is to identify the empirical sampling distribution as the natural large-horizon object for studying sampling allocation and to characterize its limiting behavior through a variational principle.
In the small-gap regime, the D-optimality criterion serves as a potential function whose maximizers determine the limiting macroscopic state.

\item \textbf{A refined asymptotic regret analysis for LinUCB in small-gap regime.}
Building on the optimal-design limit, we characterize the instance-specific asymptotic regret of LinUCB in the small-gap regime for various (finite) action sets. We show that regret grows on the $\sqrt n$ scale, with its leading constant determined by the average scaled reward gap under a D-optimal design. When the D-optimal design is unique, this identifies the coefficient of the $\sqrt{n}$  term in the regret asymptotics; when it is not unique, our result provides an asymptotic envelope of possible leading constants. For the nonunique case, we further consider the Boolean hypercube action set as an illustrative example. In this example, the coefficient of the $\sqrt n$ term can still be identified, and the resulting regret order is consistent with the minimax lower bounds~\citep{dani2007price}.

\item \textbf{Statistical inference after adaptive sampling in small-gap regime.}
We further study statistical inference based on data collected by LinUCB. Unlike classical i.i.d. data, adaptively collected observations exhibit history-dependent structure that can invalidate standard inferential procedures. We show that, nevertheless, the regularized least-squares estimator under LinUCB sampling satisfies a standard $\sqrt n$-asymptotic normality result in the small-gap regime. This contrasts with the fixed-gap analysis of LinUCB for the unit-ball action set studied by \citet{fan2025statistical}, where the convergence rate may deteriorate to $n^{1/4}$ in directions orthogonal to the true parameter. This comparison reveals that, relative to the fixed-gap regime, although the small-gap regime is statistically harder for regret minimization, it leads to a more information-efficient sampling allocation for parameter learning.

\end{itemize}

Taken together, these results explain how LinUCB's sampling allocation shapes both its regret performance and its statistical properties in the small-gap regime. The optimal-design limit governs the leading-order regret, while the asymptotic normality of the regularized least-squares estimator provides a theoretical basis for constructing asymptotically valid confidence regions from the same adaptively collected data.

\subsection{Related Works}
\textbf{\quad Sampling allocation and inference of UCB-type algorithms.}
A growing literature studies the sampling allocation of UCB-type algorithms and its implications for asymptotic inference. For multi-armed bandits, \citet{KalvitZeevi2021} analyze the sampling allocation of UCB1 in the two-armed setting and provide a complete characterization of its phase-transition behavior. \citet{KhamaruZhang2024InferenceUCB} extend this sampling allocation analysis to the multi-armed setting and establish central-limit-type results for the sample mean of each arm. \citet{fan2022typical} establish strong laws of large numbers and central limit theorems for the sampling allocation of UCB and Thompson sampling in the fixed-gap regime. Extending these results to linear bandits is challenging because actions are coupled through a shared design matrix rather than sampled independently across arms, making both the sampling allocation and the resulting inference substantially harder to analyze. Recent work by \citet{fan2025statistical} analyzes the eigenstructure of the LinUCB design matrix and develops central-limit-type inference results in the fixed-gap regime. In contrast, our work focuses on the small-gap regime and shows that the sampling allocation is governed by a fundamentally different, design-driven mechanism. This distinction also manifests itself in inference: although both regimes yield central-limit-type results, the reward-parameter estimator converges at different rates.

\textbf{Optimal experimental design.}
Optimal experimental design studies how samples should be allocated to efficiently estimate unknown parameters.
Classical sequential procedures such as the Fedorov--Wynn algorithm construct D-optimal designs \citep{Kiefer1959,AtkinsonDonevTobias2007} by repeatedly selecting design points with the largest current  variance function \citep{Wynn1970,Fedorov1972}.
Our work connects this classical design principle with LinUCB by showing that, in the small-gap regime, its sampling dynamics asymptotically follow an approximate Fedorov--Wynn-type procedure. Optimal-design principles have also been used for pure exploration in linear bandits: \citet{NIPS2014_f8d84caa} introduce a static G-optimal allocation for best-arm identification and further develop static and adaptive sampling strategies that target the estimation of reward differences.

\textbf{Sampling allocation and inference for other bandit algorithms.}
Beyond UCB-type algorithms, a parallel line of work studies Thompson sampling and related randomized policies. \citet{Han2026ThompsonSampling} characterizes the sampling allocation and downstream inference of Thompson-sampling algorithm, showing that, unlike UCB1, the sampling limits of certain arms can be stochastic rather than deterministic. This stochasticity can further induce non-Gaussian limits for the corresponding sample means. Other works modify Thompson-sampling-based procedures to improve their inferential properties while preserving regret performance\citep{DimakopoulouRenZhou2021,HalderPanKhamaru2025StableTS}. These studies emphasize that, beyond regret minimization, inference after adaptive data collection is a central consideration in bandit algorithms. This perspective identifies an algorithm's macroscopic sampling allocation as a fundamental object of study.

\textbf{Diffusion approximations and refined regret analysis in small-gap regime.}
A related line of work studies bandit algorithms in the small-gap regime, with particular attention to the case where reward gaps are of order $n^{-1/2}$ over a decision horizon $n$. This regime captures the hard instances underlying worst-case regret lower bounds and enables process-level characterizations of algorithm dynamics and refined regret analyses. \citet{KalvitZeevi2021} develop diffusion approximations for UCB1, showing that the scaled reward processes converge to Brownian motions with drift. Beyond the $n^{-1/2}$ scaling, they also identify the leading-order regret constant when reward gaps are of order $\sqrt{(\log n)/n}$ over a decision horizon $n$.   \citet{KuangWager2024WeakSignal} and \citet{FanGlynn2021DiffusionTS} establish diffusion approximations for Thompson sampling, under which the sampling process converges to the solution of a stochastic differential equation (SDE). Building on this limit, \citet{KuangWager2024WeakSignal} provide insights into the regret performance and optimality of randomized sampling algorithms, while \citet{FanGlynn2021DiffusionTS} show that the regret performance of Thompson sampling is generally insensitive to model misspecification. Together, these works provide refined characterizations of algorithm dynamics and regret performance in various small-gap regimes. However, this line of research has focused primarily on multi-armed bandits, leaving general linear bandits comparatively less understood.

\textbf{Asymptotic framework for bandit learning.}
Some recent works use diffusion-based framework to study bandit learning. This framework simplifies the underlying stochastic dynamics and enables the use of analytical tools such as SDEs and partial differential equations (PDE). \citet{Adusumilli2025RiskBanditExperiments} characterizes asymptotic Bayes and minimax risks in multi-armed bandit problems through PDEs and derives optimal policies from their numerical solutions. \citet{Lattimore2026DiffusionPolicyGradient} studies softmax policy gradient for stochastic bandits through a diffusion approximation framework. The analysis characterizes the resulting SDE dynamics and shows that regret performance is highly sensitive to the learning rate, with poorly tuned learning rates leading to linear regret on certain multi-armed instances. These works show that the asymptotic framework simplifies the analysis and allows refined characterizations of bandit algorithms beyond what regret bounds alone can capture.

\textbf{Mean-field asymptotics in high-dimensional statistics.}
Our mean-field perspective is conceptually related to recent mean-field asymptotic analyses in high-dimensional statistics. \citet{Montanari2018MeanField} shows in problems such as high-dimensional linear regression and principal component analysis, suitably normalized estimators and iterative algorithms admit precise deterministic macroscopic limits. Both this line of work and ours replace complex microscopic random dynamics by tractable macroscopic descriptions, but the limiting objects are fundamentally different.
In high-dimensional statistics, the sample size and ambient dimension typically grow together, often at a fixed aspect ratio, and the mean-field limit characterizes the asymptotic behavior of estimators or optimization algorithms. By contrast, our linear bandit model studies an adaptive sampling process over $n$ rounds where the reward gaps of order $n^{-1/2}$, the local asymptotic scaling at which signal and statistical fluctuations remain comparable.
The resulting mean-field limit characterizes the empirical sampling distribution induced by LinUCB, rather than the asymptotic distribution of an estimator.

\subsection{Outline of the Paper and Notation}
\label{subsec:outline-notation}

The rest of the paper is organized as follows.
Section~\ref{sec:preliminaries} introduces the triangular-array linear bandit framework in the small-gap regime, specifies the LinUCB algorithm, and formulates the empirical sampling distribution as the macroscopic object through which we study sampling allocation. Section~\ref{sec:main-results} establishes the optimal-design limit for LinUCB, showing that its empirical sampling distribution converges to the set of D-optimal designs. Building on this limit, it further provides a sharp asymptotic characterization of the pseudo-regret of LinUCB in the small-gap regime.
Section~\ref{sec:asymptotic-inference} studies statistical inference based on data collected by LinUCB, establishing a central-limit-type theorem for the regularized least-squares estimator in the small-gap regime.  Finally, Section~\ref{sec:conclusion} concludes the paper. 

We use the following notation throughout. Let $I_d$ denote the $d \times d$ identity matrix. For a square matrix $M\in\mathbb R^{d\times d}$, let $\operatorname{tr}(M)$ denote its trace. For a positive semidefinite matrix $M$, let $\|x\|_M=(x^\top Mx)^{1/2}$. For symmetric matrices $A$ and $B$, we write $A\succeq B$ if $A-B$ is positive semidefinite. For $a\in \mathbb R^d$, let $\delta_a$ denote the Dirac measure at $a$. Finally, $\xrightarrow{\mathbb P}$ denotes convergence in probability, and $\Rightarrow$ denotes weak convergence.

\section{Preliminaries}
\label{sec:preliminaries}

This section establishes the asymptotic framework used throughout the paper. We first formulate the linear bandit problem in a triangular-array setting, then describe the LinUCB algorithm under this convention, and finally state our objective of characterizing the limiting behavior of the empirical sampling distribution.

\subsection{Triangular-Array Linear Bandits in the Small-Gap Regime}

We first introduce the triangular-array formulation, which allows the reward parameter to vary across problem instances indexed by the decision horizon. This flexibility is essential for capturing the small-gap regime, where reward differences scale with the horizon.

We consider a sequence of linear bandit instances indexed by the horizon $n$, all sharing the same action set $\mathcal A\subset \mathbb R^d$.  In the $n$th instance, the learner interacts with the environment for $n$ rounds. At each round $t=1,\ldots,n$, the learner chooses an action
\[
A_t^{(n)}\in\mathcal A,
\]
possibly based on the past observations. The observed reward is
\begin{align}\label{eq:observation}
X_t^{(n)}=\left\langle A_t^{(n)},\theta_\star^{(n)}\right\rangle
+\eta_t^{(n)},
\end{align}
where $\theta_\star^{(n)} \in \mathbb R^d$ is the unknown reward parameter and $\eta_t^{(n)}$ is the noise at round $t$. Here we use a horizon-indexed triangular-array convention, also known as a
sequence-of-instances asymptotic regime. The
superscript ${}^{(n)}$ indexes the instance, while the subscript
$t$ indexes time within that instance. Unlike a single stochastic process
observed over increasing horizons, this formulation does not impose projective
consistency across $n$: for $m\ne n$, the variables
$X_t^{(n)}$ and $X_t^{(m)}$ need not have the same distribution.

Let $\mathcal F_{t}^{(n)}$ denote the filtration generated by the history up to time $t$. That is,
\[
\mathcal F_t^{(n)}
:=\sigma\left(A_1^{(n)},X_1^{(n)},\ldots,A_t^{(n)},X_t^{(n)}\right),
\qquad t\ge 1,
\]
with $\mathcal F_0^{(n)}$ taken to be the trivial $\sigma$-field. We impose the following condition on the noise process.

\begin{assumption}[Subgaussian noise]\label{ass:subgua_noise}
For the $n$th instance, the noise sequence
$\{\eta_t^{(n)}\}_{1\le t\le n}$
is conditionally mean-zero and uniformly $\sigma$-subgaussian given the past history and the selected action. Specifically, for every $1\le t\le n$,
\[
\mathbb E\left[\eta_t^{(n)}\mid
\mathcal F_{t-1}^{(n)},A_t^{(n)}\right]=0,
\]
and
\[
\mathbb E\left[\exp\{\tau \eta_t^{(n)}\}\mid
\mathcal F_{t-1}^{(n)},A_t^{(n)}\right]
\le\exp\left(\frac{\tau^2\sigma^2}{2}\right), \qquad \forall \tau\in\mathbb R .
\]
\end{assumption}

For technical convenience, we restrict attention to finite action sets, as formalized in the following assumption. This restriction does not affect the main message of the paper, since the action sets covered by our analysis may still be exponentially large in the dimension. For instance, the Boolean hypercube action set in
Corollary~\ref{coro:hypercube} has cardinality $2^d$.
\begin{assumption}[Finite action set]
\label{ass:bounded-action-set}
The action set $\mathcal A\subset\mathbb R^d$ is finite. Consequently, there exists a constant $L<\infty$ such that
\[
\|a\|_2 \le L,
\qquad \forall a\in\mathcal A .
\]
\end{assumption}

Since we work with a triangular-array model, the unknown reward parameter may depend on the horizon $n$. The next assumption requires this sequence of parameters to remain uniformly bounded.

\begin{assumption}[Uniformly bounded reward parameters]
\label{ass:bounded-parameters}
The sequence of reward parameters $\{\theta_\star^{(n)}\}_{n\ge 1}\subset \mathbb R^d$ is uniformly bounded. That is, there exists a constant $C_{\theta_\star}<\infty$ such that
\[
\|\theta_\star^{(n)}\|_2 \le C_{\theta_\star},
\qquad \forall n\ge 1.
\]
\end{assumption}
Finally, we impose the small-gap condition that defines the asymptotic regime studied in this paper. It requires reward differences across the action space to be of order at most $n^{-1/2}$.

\begin{assumption}[Small-gap condition]\label{asm small gap}  There exists some problem-dependent constant $C_{\mathcal{A}}$ depending on $\mathcal{A}$ but independent of $n$  such that,
\begin{align*}
    \max_{a\in\mathcal{A}} \langle a,\theta^{(n)}_{\star}\rangle-\min_{b\in\mathcal{A}} \langle b,\theta^{(n)}_{\star}\rangle\leq \frac{C_{\mathcal{A}}}{\sqrt{n}},\qquad \forall n\ge 1.
\end{align*}
\end{assumption}
Assumption~\ref{asm small gap} states that the reward range over the action set is uniformly of order at most $n^{-1/2}$. Thus, in the $n$th bandit instance, the largest reward gap between any two actions is no greater than the natural $n^{-1/2}$ estimation uncertainty over a horizon of length $n$. This scaling captures the statistically hardest regime for learning: the reward gaps are small enough that actions cannot be reliably distinguished without substantial exploration. It is precisely this type of small-gap construction that underlies minimax regret lower bounds in bandit problems; see, for example, \citet{lattimore2020bandit}.
The condition is related to the small-gap regimes studied
in multi-armed bandits, including the two-arm setting of
\citet{KalvitZeevi2021} and the $K$-armed setting of
\citet{KhamaruZhang2024InferenceUCB}. Our formulation imposes a stronger
global condition, requiring the entire reward range over the action set
to shrink at the rate $n^{-1/2}$.

The preceding setup and assumptions define the triangular-array linear bandit framework in the small-gap regime. We next introduce the classical LinUCB algorithm within this framework.

\subsection{The LinUCB Algorithm}

We now specify the LinUCB algorithm within the triangular-array framework, including its regularized design matrix, regularized least-squares estimator, and action selection rule.
 Fix a regularization parameter $\lambda>0$. For the $n$th problem instance, define 
\[
V_t^{(n)}=\lambda I_d+\sum_{s=1}^t A_s^{(n)}(A_s^{(n)})^\top, \qquad 0\le t\le n,\]
the regularized design matrix accumulated up to time $t$.
The corresponding regularized least-squares estimator is
\begin{align}\label{eq:least_square}
\widehat \theta_t^{(n)}=\bigl(V_t^{(n)}\bigr)^{-1}
\sum_{s=1}^t A_s^{(n)} X_s^{(n)}, \qquad 0\le t\le n,
\end{align}
where the summation is understood to be zero when $t=0$.

At each round $t=1,\ldots,n$, LinUCB selects an action according to
\[
A_t^{(n)}\in\argmax_{a\in\mathcal A}\left\{
a^\top \widehat\theta_{t-1}^{(n)}+\sqrt{\beta_t}  
\|a\|_{\left(V_{t-1}^{(n)}\right)^{-1}}\right\},
\]
where 
\begin{align}
\label{def:beta_t}
\sqrt{\beta_t}=\sqrt{\lambda}C_{\theta_\star}+\sqrt{2\log t+ d\log\left(1+\frac{t L^2}{d\lambda}\right)},\qquad 1\le t\le n.
\end{align}
This choice corresponds to the standard confidence radius used in linear bandits; see, for example, \citet{lattimore2020bandit}. The analysis in this paper does not rely on the precise expression in \eqref{def:beta_t}. It only requires that the exploration radius $\sqrt{\beta_t}$ is deterministic, nondecreasing in $t$, and satisfies, for fixed $d,L,\lambda,\sigma$,
\[
\sqrt{\beta_t}=\Omega(\sqrt{\log t}),
\]
or equivalently, $\beta_t=\Omega(\log t)$.

\subsection{Empirical Sampling Distributions}

We aim to characterize LinUCB’s sampling allocation in the small-gap regime. Such
a characterization reveals how information accumulates, explains the mechanisms behind regret guarantees, and supports refined asymptotic analysis. Analyzing this behavior is challenging because each sampling decision is history-dependent, feedback-driven, and coupled with other decisions through a shared linear reward model. We adopt a mean-field perspective: rather than tracking individual adaptive decisions, we study the empirical sampling distribution as a macroscopic object that summarizes the algorithm's average behavior.

For $1\le t\le n$,  define 
\[
\pi_t^{(n)}:=\frac{1}{t}\sum_{s=1}^t \delta_{A_s^{(n)}}
\in \mathcal P(\mathcal A)
\]
the empirical sampling distribution up to time $t$,
where $\mathcal P(\mathcal A)$ denotes the set of Borel probability measures on $\mathcal A$. Our goal is therefore to characterize the limiting behavior of the full-horizon empirical sampling distribution $\pi_n^{(n)}$ as $n\to\infty$.

To formulate this limiting behavior, we endow $\mathcal P(\mathcal A)$ with a metric.  For $\mu,\nu\in\mathcal P(\mathcal A)$, define the $1$-Wasserstein distance by
\[
W_1(\mu,\nu):=
\sup_{f: \mathcal{A}\to\mathbb{R},~\mathrm{Lip}(f)\le 1}
\left|\int_{\mathcal A} f(a)  \mu({\rm d} a)
-\int_{\mathcal A} f(a)  \nu({\rm d} a)\right|.
\]
Since $\mathcal A$ is compact under Assumption~\ref{ass:bounded-action-set}, $W_1$ metrizes weak convergence on $\mathcal P(\mathcal A)$; see, e.g., \citet[Theorem~6.9]{Villani2009OptimalTransport}. Moreover, for any nonempty set $\mathcal P\subseteq \mathcal P(\mathcal A)$, define
\[
\operatorname{dist}_{W_1}(\mu,\mathcal P):= \inf_{\nu\in\mathcal P} W_1(\mu,\nu).
\]
This allows us to formulate convergence to a set that may not be a singleton.

Within the triangular-array framework in the small-gap regime, the mean-field formulation identifies the empirical sampling distribution as the central macroscopic object. The next section uses this framework to analyze the limiting behavior of $\pi_n^{(n)}$ as $n\to\infty$.

\section{Optimal-Design Limit and Regret Analysis of LinUCB}
\label{sec:main-results}

In this section, we study the sampling behavior and regret performance of LinUCB in the small-gap regime. 
We first identify the limiting empirical sampling distribution and show that it is governed by a classical optimal-design principle. 
Building on this structural result, we then obtain a refined asymptotic characterization of LinUCB's pseudo-regret in the same regime.

 The small-gap regime is particularly relevant for regret analysis because it captures the hard instances underlying minimax regret lower bounds in linear bandits. Since LinUCB is known to achieve near optimal minimax regret up to logarithmic factors, studying both the sampling distribution and the resulting regret in this regime provides a natural way to understand its worst-case performance. 

\subsection{Optimal-Design Limit in the Small-Gap Regime}
\label{subsec:optimal-design-limit}
We now identify the asymptotic structure of LinUCB's sampling allocation. 
Our main result shows that, despite the history-dependent nature of its sequential decisions, the empirical sampling distribution converges to the set of D-optimal designs. 

We begin by recalling the classical optimal experimental design problem, which seeks to allocate samples across a prescribed design space to estimate an unknown parameter as efficiently as possible~\citep{Kiefer1959}. 
In our linear setting~\eqref{eq:observation}, the design space is the action set $\mathcal A$, and the unknown parameter in the $n$th instance is $\theta_\star^{(n)}$. A design is represented by a probability distribution $\pi\in\mathcal P(\mathcal A)$ specifying the sampling allocation over the action set, with associated information matrix
\begin{align}\label{def:info_matrix}
M(\pi):=\int_{\mathcal A} aa^\top \pi({\rm d} a).
\end{align}
A classical criterion is D-optimality, which chooses the design to maximize the determinant of the information matrix. For a fixed action set $\mathcal A$, this criterion is independent of the unknown parameter $\theta_\star^{(n)}$.  The D-optimality criterion is given by
\[
\Phi(\pi)=\log\det M(\pi)
\]
with the convention that $\Phi(\pi)=-\infty$ whenever $M(\pi)$ is singular. The set of D-optimal designs is then defined as
\[
\mathcal P_D:=\argmax_{\pi\in\mathcal P(\mathcal A)}
\Phi(\pi).
\]
The D-optimality criterion has a natural statistical interpretation. 
Under a fixed design with sampling proportions $\pi$ and independent Gaussian errors of common variance, the covariance matrix of the least-squares estimator is proportional to $M(\pi)^{-1}$, provided that $M(\pi)$ is nonsingular. For a fixed sample size and confidence level, maximizing $\det M(\pi)$ therefore minimizes the determinant of the covariance matrix and, geometrically, the volume of the associated confidence ellipsoid for the unknown parameter~\citep{AtkinsonDonevTobias2007}. In this sense, D-optimality reduces overall estimation uncertainty.

We now state the main convergence theorem, which characterizes the limiting behavior of the empirical sampling distribution generated by LinUCB in the small-gap regime.

\begin{theorem}[Optimal-design limit of LinUCB in the small-gap regime]
\label{thm:linucb-general-D-optimal}
Suppose that Assumptions~\ref{ass:subgua_noise}--\ref{asm small gap} hold and that
$\operatorname{span}(\mathcal A)=\mathbb R^d$. Then as $n \to \infty$, the empirical sampling distribution $\pi_n^{(n)}$ generated by LinUCB satisfies
\[
\operatorname{dist}_{W_1}\left(\pi_n^{(n)},\mathcal P_D\right)
\xrightarrow{\mathbb P}0.
\]
In particular, if the D-optimal design is unique, that is,
$\mathcal P_D=\{\pi^\star\}$, then
\[
\pi_n^{(n)}\Rightarrow \pi^\star \qquad\text{in probability}.
\]
\end{theorem}

\begin{remark}[Set-valued design limit]
The uniqueness of the D-optimal design ensures a deterministic limiting design. Without uniqueness, Theorem~\ref{thm:linucb-general-D-optimal} only implies convergence to the set $\mathcal P_D$, not to a particular element of that set.
\end{remark}

This theorem yields, in the special case of $K$-armed bandits, the same
uniform-allocation limit as that established for UCB1 in the small-gap
regimes studied by \citet{KalvitZeevi2021,KhamaruZhang2024InferenceUCB}.
When the $K$ arms are represented by the canonical basis vectors, the D-optimal design over this action set is the uniform allocation over the arms. Hence, the uniform sampling limit arises as a special case of Theorem~\ref{thm:linucb-general-D-optimal}. 

Theorem~\ref{thm:linucb-general-D-optimal} shows that the empirical sampling distribution of LinUCB converges to the set of D-optimal designs in the $W_1$-distance.  Thus, in the small-gap regime, LinUCB does not simply concentrate its sampling on the optimal actions. Instead, its sampling frequencies converge to an optimal design for learning the unknown reward parameter, in the sense of minimizing the volume of the associated confidence ellipsoid. By the equivalence theorem of~\citet{Kiefer_Wolfowitz_1960}, this limit also has a G-optimal interpretation: LinUCB asymptotically minimizes the largest variance of the estimated reward over the action space. Hence, in the small-gap regime, LinUCB does more than minimize regret; it implicitly constructs an information-efficient sampling  design, thereby linking regret-driven online learning with information-efficient experimental design.

This result also reveals a phenomenon complementary to the ``end-of-optimism'' fixed-gap counterexample in Section~8 of \citet{lattimore2017end}. In that example, the action set is $$\mathcal A=\{e_1,e_2,x\}\subset\mathbb R^2,$$
where $e_1=(1,0)$, $e_2=(0,1)$, and $x=(1-\epsilon,8\alpha\epsilon)$ for small $\epsilon>0$. The true reward parameter is $\theta_\star=e_1$, so $e_1$ is the unique optimal action, while the gaps satisfy $\Delta_{e_2}=1$ and $\Delta_x=\epsilon$. Thus, $x$ is nearly optimal and close to the optimal direction, whereas $e_2$ is clearly suboptimal but provides information in the direction orthogonal to $\theta_\star$. The fixed-gap example shows that LinUCB can be suboptimal by an arbitrarily large constant factor because it quickly rules out $e_2$ and stops sampling it, even though $e_2$ provides highly valuable information that an optimal algorithm should use it for exploration.  Theorem~\ref{thm:linucb-general-D-optimal} shows that in the (hard) small-gap regime, LinUCB can no longer rule out such actions quickly. Instead, the D-optimal design limit reveals an information-efficient exploration mechanism. This provides a new perspective on how LinUCB achieves the optimal regret rate in this hard small-gap regime.

This theorem also provides a mean-field characterization of LinUCB in the small-gap regime; see Table~\ref{tab:meanfield_linucb}. The empirical sampling distribution $\pi_n^{(n)}$ plays the role of a macroscopic state, while the D-optimality criterion $\Phi(\pi)$ acts as the potential function that determines its limiting behavior. Thus, when the reward differences are of order at most $n^{-1/2}$, the limiting sampling distribution is characterized by the maximizers of the potential function.

\subsection{Regret Analysis in the Small-Gap Regime}

We next use the optimal-design limit to shed light on the regret behavior of LinUCB in the small-gap regime.
Recall that $A_t^{(n)}$ is the action selected by LinUCB at time $t$ in the
$n$th problem instance. We define the {\em  cumulative pseudo-regret} (or simply {\em pseudo-regret}) by
\[
\mathcal R_n:=\sum_{t=1}^n
\left(\max_{b\in\mathcal A}\langle b,\theta_\star^{(n)}\rangle
-\langle A_t^{(n)},\theta_\star^{(n)}\rangle
\right).
\]
As a special case of the small-gap regime, suppose that
\begin{align}\label{theta_scale}
\theta_\star^{(n)}=\frac{\theta_0}{\sqrt n}.
\end{align}
For each $a\in\mathcal A$, define the scaled reward gap
\begin{align}\label{eq:delta}
\Delta(a):=
\max_{b\in\mathcal A}\langle b,\theta_0\rangle
-\langle a,\theta_0\rangle .
\end{align}
The following proposition characterizes the pseudo-regret of LinUCB in the small-gap regime. The proof is deferred to Appendix~\ref{appsub:proofs-coro}.
For $x\in\mathbb R$ and a nonempty set $S\subset\mathbb R$, we write
\[
\operatorname{dist}(x,S):=\inf_{y\in S}|x-y|,
\]
for the distance from the point $x$ to the set $S$.
\begin{proposition}[Pseudo-regret in the small-gap regime]
\label{prop:small-gap-regret-limit}
Suppose that the assumptions of Theorem~\ref{thm:linucb-general-D-optimal}
hold. Then the pseudo-regret of LinUCB satisfies 
\[
\operatorname{dist}
\left(
\frac{\mathcal R_n}{\sqrt n},
\left\{
\int_{\mathcal A} \Delta(a)\,\pi({\rm d}a):
\pi\in\mathcal P_D
\right\}
\right)
\xrightarrow{\mathbb P}0 .
\]
In particular, if the D-optimal design is unique, i.e.,
$\mathcal P_D=\{\pi^\star\}$, then
\[
\frac{\mathcal R_n}{\sqrt n}
\xrightarrow{\mathbb P}
\int_{\mathcal A} \Delta(a)\,\pi^\star({\rm d}a).
\]
\end{proposition}

Proposition~\ref{prop:small-gap-regret-limit} provides a sharp, instance-specific asymptotic characterization of the pseudo-regret of LinUCB for various (finite) action sets in the small-gap regime. It shows that the $\sqrt{n}$-normalized pseudo-regret $\mathcal R_n/\sqrt n$ is determined by the average scaled reward gap under a D-optimal design. When the D-optimal design is unique, this identifies the coefficient of the  $\sqrt{n}$  term in the regret asymptotics. When it is not unique, the proposition instead provides an asymptotic envelope. Pinning down the exact constant would require identifying which D-optimal design is selected by the adaptive dynamics, a finer question that may require a fluid- or even diffusion-level characterization of the sampling dynamics~\citep{KalvitZeevi2021,Lattimore2026DiffusionPolicyGradient}. We leave this characterization as an open problem.

Proposition \ref{prop:small-gap-regret-limit} may be viewed   as an asymptotically sharp refinement of the classical
worst-case regret guarantee for LinUCB. Standard analyses show that LinUCB
achieves $\widetilde O(d\sqrt n)$ regret uniformly over bounded linear bandit
instances; see, e.g., \citet{abbasi2011improved}. Our analysis focuses on the
small-gap regime, where the reward gaps scale as $n^{-1/2}$, the canonical
scaling of the hardest instances underlying minimax regret lower bounds of order $\sqrt{n}$. Within this regime, Proposition~\ref{prop:small-gap-regret-limit} identifies the leading-order constant in the regret asymptotics.

When the D-optimal design is not unique, Proposition~\ref{prop:small-gap-regret-limit} may still completely characterize the coefficient of the leading-order $\sqrt{n}$ term. We illustrate this point with the Boolean hypercube action set through the following corollary.

\begin{corollary}[Pseudo-regret of LinUCB on the Boolean hypercube matches the known lower bound]
\label{coro:hypercube}
Consider the Boolean hypercube action set
\[
\mathcal A=\{0,1\}^d
\]
and suppose that
\[
\theta_{\star}^{(n)}
\in
\left\{-\frac{1}{\sqrt n},\frac{1}{\sqrt n}\right\}^d,
\]
so that $\theta_{\star}^{(n)}=\theta_0/\sqrt n$ for some
$\theta_0\in\{-1,1\}^d$. Then, for every fixed $d\geq 2$ and every sign pattern $\theta_0$, there exists an explicit constant $C_d(\theta_0)\in [d/4,3d/4]$ (defined in Eqn.~\eqref{eq:C_d}), depending only on $d$ and $\theta_0$, that characterizes the leading-order asymptotic pseudo-regret of LinUCB as $n\to\infty$:
\[
\frac{\mathcal R_n}{\sqrt n}
\xrightarrow{\mathbb P} C_d(\theta_0).
\]
 Consequently, the pseudo-regret is of order $\Theta(d\sqrt n)$. Both the
hypercube construction with $\theta_0\in\{-1,1\}^d$ and the resulting
fixed-dimensional asymptotic regret order are consistent with the lower
bound in Theorem~4.3 of~\citet{dani2007price}.
\end{corollary}

\begin{proof}[Proof of Corollary~\ref{coro:hypercube}]
We first identify the D-optimal design on $\mathcal A$. Let
\begin{align}\label{eq:c_d}
c_d:=
\begin{cases}
\dfrac{d+1}{4d}, & d \text{ odd},\\[6pt]
\dfrac{d+2}{4(d+1)}, & d \text{ even}. 
\end{cases}
\end{align}
The D-optimal information matrix is
\begin{align}\label{eq:information matrix}
M^\star=c_d(I_d+\mathbf 1\mathbf 1^\top).
\end{align}
This D-optimal information matrix is attainable by a sampling distribution; details of its attainability are provided in Appendix~\ref{appsub:proofs-coro}.
 Since $\Phi(\cdot)$ is strictly concave in the information
matrix, the D-optimal information matrix is unique. Thus, for every
$\pi\in\mathcal P_D$,
\[
M(\pi)=M^\star.
\]
We now compute the regret limit constant. 
By~\eqref{eq:delta}, for every $\pi\in\mathcal P_D$,
\begin{align*}
\int_{\mathcal A}\Delta(a)\pi({\rm d}a)=\sum_{i=1}^d(\theta_{0,i})_+
-\int_{\mathcal A}\langle a,\theta_0\rangle\pi({\rm d}a)  =\sum_{i=1}^d(\theta_{0,i})_+
-\sum_{i=1}^d\theta_{0,i}\int_{\mathcal A}a_i\pi({\rm d}a),
\end{align*}
where $(x)_+$ denotes the positive part of $x$, i.e.,  $\max\{x,0\}$.
Using $a_i^2=a_i$ for $a\in\{0,1\}^d$, we obtain
\begin{align}\label{prob_aa}
\int_{\mathcal A} a_i\,\pi({\rm d}a)
=\int_{\mathcal A} a_i^2\,\pi({\rm d}a)
=(M^\star)_{ii}
=2c_d.
\end{align}
Then for every D-optimal design $\pi\in\mathcal P_D$,
\[
\int_{\mathcal A} \Delta(a) \pi({\rm d}a)
=\sum_{i=1}^d(\theta_{0,i})_+
-2c_d\sum_{i=1}^d\theta_{0,i}: =C_d(\theta_0).
\]
Equivalently,
\begin{equation}
C_d(\theta_0)=\frac12\|\theta_0\|_1+\frac{1-4c_d}{2}{\bf 1}^\top \theta_0.     \label{eqn:def_Cd}
\end{equation}
Applying Proposition~\ref{prop:small-gap-regret-limit}, 
\begin{align}\label{eq:regret_limit}
\frac{\mathcal R_n}{\sqrt n}
\xrightarrow{\mathbb P}
C_d(\theta_0),
\end{align}
which characterizes the coefficient of the leading-order $\sqrt{n}$ term in the pseudo-regret of LinUCB.

In particular, when $\theta_{\star}^{(n)}\in\{-1/{\sqrt{n}},1/{\sqrt{n}}\}^d$, we have $\theta_{0}\in\{-1,1\}^d$.  Let $S=\{i \in [d]:\theta_{0,i}>0\}$, then
\begin{align}\label{eq:C_d}
C_d(\theta_0)=
(1-2c_d)|S|+2c_d(d-|S|).
\end{align}

Moreover, consistent with the intuition that $c_d\approx 1/4+1/(4d)$ and $|S|\leq d$ regardless of $d$ and $S$, substituting the exact values of $c_d$ from~\eqref{eq:c_d} into the expression for $C_d(\theta_0)$ in~\eqref{eq:C_d} shows that, for every $d\geq 1$ and every sign pattern $\theta_0\in\{-1,1\}^d$,
\[
\left|C_d(\theta_0)-\frac{d}{2}\right|\leq \frac{1}{2}.
\]
Thus, for every fixed $d\geq 2$,
\begin{align}\label{eq:fixed}
\frac{d}{4}\leq C_d(\theta_0)\leq \frac{3d}{4}.
\end{align}
Hence, uniformly over all sign patterns, the leading-order regret
constant $C_d(\theta_0)$ is of order $d$.  By~\eqref{eq:regret_limit}, the resulting pseudo-regret
order $\Theta(d\sqrt n)$ is consistent, for every fixed $d\geq 2$, with
the lower bound in Theorem~4.3 of~\citet{dani2007price}.
\end{proof}

\begin{remark}[Interpretation of $C_d(\theta_0)$]
    The constant $C_d(\theta_0)$ in~\eqref{eq:C_d} admits a simple coordinatewise interpretation. Under any D-optimal design $\pi$, by~\eqref{prob_aa}, 
\[
\mathbb P_{\pi}(a_i=1)=2c_d.
\]
If $\theta_{0,i}>0$, the optimal action has $a_i=1$, so regret is incurred whenever the sampled action has $a_i=0$, an event with probability $1-2c_d$. Conversely, if $\theta_{0,i}<0$, the optimal action has $a_i=0$, so regret is incurred whenever the sampled action has $a_i=1$, an event with probability $2c_d$. Therefore, $C_d(\theta_0)$ is precisely the expected number of coordinates on which an action drawn from the D-optimal design differs from the optimal action. 
\end{remark}

\begin{remark}[Exact large-dimensional
description]The same calculation also yields the uniform large-dimensional
description
\[
C_d(\theta_0)=\frac{d}{2}+O(1)
\qquad\text{as }d\to\infty.
\]
This last relation describes the dependence of the fixed-$d$ limiting
constant on $d$; it should not be interpreted as a joint asymptotic
regret result with $d=d_n\to\infty$. Such a result would require a separate analysis under an explicitly specified joint scaling of $d_n$ and $n$ and is neither needed for, nor implied by, the fixed-dimensional statement~\eqref{eq:fixed}, which already matches the targeted known lower bound.
\end{remark}

 This behavior is in stark contrast to that in the  fixed-gap regime. When the parameter $\theta_\star$ is fixed, LinUCB eventually identifies the
coordinatewise optimal action $a^\star = \mathbf{1}\{\theta_0>0\}$ with high confidence and allocates nearly all
samples to it, so the empirical sampling distribution converges to the point
mass on the optimal action $\delta_{a^\star}$. By contrast, in the small-gap regime $\theta_\star^{(n)}=\theta_0/\sqrt n$, the reward gaps remain comparable to the statistical uncertainty. Consequently, LinUCB cannot confidently identify the coordinatewise optimal action and therefore continues to sample actions that differ from it in individual coordinates. The limiting D-optimal design thus spreads its sampling effort over many actions rather than concentrating on the optimum.


\subsection{Main Steps in the Proof of Theorem~\ref{thm:linucb-general-D-optimal}}
\label{subsec:main-steps-proof}

In this subsection, we outline the proof of Theorem~\ref{thm:linucb-general-D-optimal}. The main technical idea is a two-step reduction: we first show that, after a negligible initial period, the adaptive LinUCB selection rule behaves like approximate variance maximization; we then establish an approximate Fedorov--Wynn convergence result \citep{Wynn1970,Fedorov1972}, which identifies the optimal-design limit.

The complete argument is structured into four sequential steps.
First, we establish a uniform prediction-error bound that controls the reward estimation error uniformly over time and actions. Second, we use this bound to show that, after a negligible initial period, LinUCB behaves as an approximate variance maximization procedure. Third, we prove an approximate Fedorov--Wynn convergence result for triangular arrays, showing that approximate variance maximization drives the empirical sampling distribution toward the set of D-optimal designs. Finally, we verify that the initial period has asymptotically negligible mass, thereby obtaining the desired design limit. Comprehensive proofs for each step are deferred to Appendices~\ref{app:proof-concentration-bound}--\ref{app:late-stage-approx-fw}.

\subsubsection*{Step 1: Uniform prediction-error control}

The following lemma establishes a uniform prediction-error bound for the regularized least-squares estimator under the triangular-array model. The bound holds simultaneously over all rounds and all actions.

\begin{lemma}[Uniform prediction-error bound for finite action sets]
\label{lem:general-linear-concentration}
Suppose that Assumptions~\ref{ass:subgua_noise}-\ref{ass:bounded-parameters} hold. Under the LinUCB algorithm, there exists a sequence of events $\{\mathcal E_n\}_{n\ge 1}$ such that
\[
\lim_{n\rightarrow \infty} \mathbb P(\mathcal E_n)=1,
\]
and on $\mathcal E_n$, for all $1\le t\le n$ and all $a\in \mathcal{A}$,
\begin{align}\label{ieq:estimate_error}
\left|\left\langle\widehat\theta_t^{(n)}-\theta_{\star}^{(n)},a
\right\rangle\right|\le C_n\left\|a\right\|_{(V_t^{(n)})^{-1}},
\end{align}
where 
\begin{align}\label{def_cn}
C_n=\sqrt{\lambda}C_{\theta_\star} + (2\sigma+1)\sqrt{|\mathcal A|\log\log n}.
\end{align}
\end{lemma}

Lemma~\ref{lem:general-linear-concentration} provides uniform control over the prediction error scaled by the LinUCB confidence radius. Because $C_n = \Theta(\sqrt{\log\log n})$, this concentration bound is sufficiently sharp to allow a precise comparison between the estimated reward and exploration components of the LinUCB index.

\subsubsection*{Step 2: Reduction to approximate variance maximization}
In the small-gap regime, the true reward differences across actions are of order at most $n^{-1/2}$. The uniform prediction-error control from Lemma~\ref{lem:general-linear-concentration} further ensures that the reward estimation error fluctuates at a strictly lower order than the exploration bonus. Consequently, after a negligible initial period, the LinUCB selection rule is governed primarily by the variance-based confidence radius. The following proposition formalizes this reduction.

\begin{proposition}[LinUCB as approximate variance maximization]
\label{prop:linucb-approx-fw}
Suppose that Assumptions~\ref{ass:subgua_noise}--\ref{asm small gap} hold. Let $\{\mathcal{E}_n\}_{n\ge 1}$ be the sequence of events defined in Lemma~\ref{lem:general-linear-concentration}. Then there exist a deterministic sequence $\varepsilon_n \to 0$ and an integer $n_0 \ge 1$ such that, for all $n \ge n_0$, on the event $\mathcal{E}_n$, the action selection rule induced by LinUCB satisfies the following inequality uniformly for all $r_n \le t \le n$:
\begin{equation}
\label{approximate_maximizer}
\left(A_t^{(n)}\right)^\top \left(V_{t-1}^{(n)}\right)^{-1} A_t^{(n)} \ge (1-\varepsilon_n) \max_{b \in \mathcal{A}} b^\top \left(V_{t-1}^{(n)}\right)^{-1} b,
\end{equation}
where the initialization window threshold is given by
\begin{equation}
\label{def_rn}
r_n := \left\lceil \exp\left\{(\log\log n)^2\right\} \right\rceil.
\end{equation}
\end{proposition}

Proposition~\ref{prop:linucb-approx-fw} formalizes the intuition that after a negligible initial period (noting that $r_n = o(n)$), the exploration-exploitation trade-off of LinUCB simplifies. Specifically, the action selected by the algorithm at each round asymptotically tracks an approximate maximizer of the variance criterion $a^\top (V_{t-1}^{(n)})^{-1}a$. Consequently, during the later stages of the interaction (i.e., when $t\approx n$), LinUCB behaves like an approximate Fedorov--Wynn algorithm.

\subsubsection*{Step 3: Approximate Fedorov--Wynn convergence}

We next establish an approximate version of the Fedorov--Wynn convergence theorem tailored to triangular arrays. Specifically, we demonstrate that any action sequence asymptotically tracking a variance-maximization rule must yield an empirical distribution that converges to the set of D-optimal designs.

\begin{proposition}[Approximate Fedorov--Wynn convergence for triangular arrays]
\label{prop:fw-triangular-array-robustness}
Suppose that Assumption~\ref{ass:bounded-action-set} holds and that $\operatorname{span}(\mathcal{A}) = \mathbb{R}^d$. Let $\{m_n\}_{n\ge 1}$ be a sequence of integers satisfying $m_n \to \infty$ as $n \to \infty$. For each $n \ge 1$, let $Z_1^{(n)}, \ldots, Z_{m_n}^{(n)} \in \mathcal{A}$, and let the initial matrix $B_0^{(n)} \succ 0$ satisfy
\[
B_0^{(n)} \succeq \lambda I_d \qquad \text{and} \qquad \lim_{n \to \infty} \frac{\operatorname{tr}\big(B_0^{(n)}\big)}{m_n} = 0
\]
for some fixed constant $\lambda > 0$. For each $n \ge 1$, define the sequence of matrices
\[
B_s^{(n)} := B_0^{(n)} + \sum_{u=1}^s Z_u^{(n)}\bigl(Z_u^{(n)}\bigr)^\top, \qquad s = 0, \ldots, m_n.
\]
Suppose there exists a deterministic sequence $\delta_n \to 0$ such that, for all $s = 1, \ldots, m_n$, the following inequality holds:
\begin{equation}
\label{con_appro_large}
\bigl(Z_s^{(n)}\bigr)^\top\bigl(B_{s-1}^{(n)}\bigr)^{-1} Z_s^{(n)} \ge (1-\delta_n) \max_{a \in \mathcal{A}} a^\top\bigl(B_{s-1}^{(n)}\bigr)^{-1} a.
\end{equation}
Let $\bar{\pi}_{m_n}^{(n)} := \frac{1}{m_n} \sum_{s=1}^{m_n} \delta_{Z_s^{(n)}} \in \mathcal{P}(\mathcal{A})$ denote the empirical distribution of the sequence. Then
\[
\lim_{n\to \infty}\operatorname{dist}_{W_1}\bigl(\bar\pi_{m_n}^{(n)},\mathcal P_D\bigr)=0.
\]
\end{proposition}

Proposition~\ref{prop:fw-triangular-array-robustness} extends the classical Fedorov--Wynn convergence mechanism to an approximate triangular-array framework. It guarantees that the D-optimal design limit is fully preserved under approximate variance maximization, provided that the relative trace contribution of the initialization matrix $B_0^{(n)}$ vanishes asymptotically. Given the prevalence of tracking procedures in sequential analysis, this robustness result may be of independent mathematical interest.

\subsubsection*{Step 4: Removing the negligible initial period}

The following result shows that the approximate Fedorov--Wynn argument remains valid when the approximate variance-maximization condition holds only after a vanishing initial period.

\begin{lemma}[Late-stage approximate Fedorov--Wynn convergence]
\label{lem:late-stage-approx-fw}
Suppose that Assumption~\ref{ass:bounded-action-set} holds and that $\operatorname{span}(\mathcal{A}) = \mathbb{R}^d$. For each $n \ge 1$, consider a sequence of actions $A_1^{(n)}, \ldots, A_n^{(n)} \in \mathcal{A}$. Define the regularized design matrices
\[
B_t^{(n)} := \lambda I_d + \sum_{u=1}^t A_u^{(n)}\bigl(A_u^{(n)}\bigr)^\top, \qquad 0 \le t \le n
\]
for some fixed constant $\lambda > 0$, and the empirical action distributions
\[
\pi_t^{(n)} := \frac{1}{t} \sum_{u=1}^t \delta_{A_u^{(n)}} \in \mathcal{P}(\mathcal{A}), \qquad 1 \le t \le n.
\]
Suppose there exist deterministic sequences $\{r_n\}_{n\ge 1}$ and $\{\varepsilon_n\}_{n\ge 1}$ satisfying
\[
\lim_{n \to \infty} r_n = \infty, \qquad \lim_{n \to \infty} \frac{r_n}{n} = 0, \qquad \lim_{n \to \infty} \varepsilon_n = 0,
\]
and an integer $n_0 \ge 1$ such that, for all $n \ge n_0$ and all $r_n \le t \le n$, the following approximate variance-maximization condition holds:
\begin{equation}
\label{con:approximate-large-infinite}
\bigl(A_t^{(n)}\bigr)^\top\bigl(B_{t-1}^{(n)}\bigr)^{-1} A_t^{(n)} \ge (1-\varepsilon_n) \max_{a \in \mathcal{A}} a^\top \bigl(B_{t-1}^{(n)}\bigr)^{-1} a.
\end{equation}
Then,
\[
\lim_{n \to \infty} \operatorname{dist}_{W_1}\bigl(\pi_n^{(n)}, \mathcal{P}_D\bigr) = 0.
\]
\end{lemma}

We are now ready to prove the main result by combining the preceding lemmas and propositions. On the high-probability event from
Lemma~\ref{lem:general-linear-concentration},
Proposition~\ref{prop:linucb-approx-fw} shows that LinUCB satisfies approximate
variance maximization after a negligible initial period. The deterministic
late-stage convergence result in Lemma~\ref{lem:late-stage-approx-fw} then
implies convergence of the empirical sampling distribution to the
set of D-optimal designs.

\begin{proof}[Proof of Theorem~\ref{thm:linucb-general-D-optimal}]
By Lemma~\ref{lem:general-linear-concentration}, there exists a sequence of high-probability events $\{\mathcal{E}_n\}_{n\ge 1}$ satisfying
\begin{equation}
\label{event_1}
\lim_{n \to \infty} \mathbb{P}(\mathcal{E}_n) = 1,
\end{equation}
on which the uniform least-squares prediction error bound \eqref{ieq:estimate_error} holds.

Next, Proposition~\ref{prop:linucb-approx-fw} ensures that on each event $\mathcal{E}_n$, after the negligible initial window $r_n = o(n)$, the actions chosen by LinUCB approximately maximize the variance criterion in the sense of \eqref{approximate_maximizer}. Consequently, on $\mathcal{E}_n$, the sample path realization of the action sequence satisfies the conditions of Lemma~\ref{lem:late-stage-approx-fw}.

Applying Lemma~\ref{lem:late-stage-approx-fw} to the total empirical action distribution
\[
\pi_n^{(n)} = \frac{1}{n} \sum_{t=1}^n \delta_{A_t^{(n)}},
\]
it follows that on the event $\mathcal{E}_n$,
\[
\lim_{n \to \infty} \operatorname{dist}_{W_1}\left(\pi_n^{(n)}, \mathcal{P}_D\right) = 0.
\]
Combining this sample-path convergence with the probability limit in \eqref{event_1} establishes that
\[
\operatorname{dist}_{W_1}\left(\pi_n^{(n)}, \mathcal{P}_D\right) \xrightarrow{\mathbb{P}} 0.
\]

In the case where the D-optimal design is unique, i.e., $\mathcal{P}_D = \{\pi^\star\}$, this convergence simplifies to
\[
W_1\left(\pi_n^{(n)}, \pi^\star\right) \xrightarrow{\mathbb{P}} 0.
\]
Because the action space $\mathcal{A}$ is compact under Assumption~\ref{ass:bounded-action-set}, convergence in the $1$-Wasserstein metric is equivalent to weak convergence. We conclude that
\[
\pi_n^{(n)} \Rightarrow \pi^\star \qquad \text{in probability},
\]
which completes the proof.
\end{proof}

The proof of Theorem~\ref{thm:linucb-general-D-optimal} can therefore be summarized as a sequence of reductions. The uniform concentration bound controls the reward estimation error; together with the small-gap assumption, this reduces the LinUCB policy, after a brief initialization window, to approximate variance maximization. The approximate Fedorov--Wynn convergence result then converts this local variance maximization property into global convergence toward the set of D-optimal designs.

\section{Statistical Inference of Adaptively Collected LinUCB Data}
\label{sec:asymptotic-inference}

This section moves beyond the sampling dynamics of LinUCB to study statistical inference based on the data generated by the algorithm. In contrast to classical i.i.d. sampling environments, observations generated by LinUCB are collected adaptively, with each action selection depending on previous decisions and observed rewards. This sequential dependence makes inference delicate, since the information accumulated by the algorithm is itself random and path-dependent.

The optimal-design limit established in Theorem~\ref{thm:linucb-general-D-optimal} provides the key structural simplification needed to address this challenge. Building on this limit, we derive asymptotic normality for the regularized least-squares estimator, which justifies statistical inference based on the adaptively collected data.

\subsection{Asymptotic Normality in the Small-Gap Regime}
\label{subsec:asymptotic-normality}

To establish the limiting distribution of the regularized least-squares estimator, we impose an additional mild condition on the noise process. In addition to the uniform conditional subgaussian property required in Assumption~\ref{ass:subgua_noise}, we further assume that the conditional noise variance remains constant within each $n$-horizon problem instance and converges to a positive limit. This limiting variance serves as the scaling parameter that determines the asymptotic covariance structure in our main inference result.

\begin{assumption}[Conditional Noise Variance]
\label{ass:conditional-noise-variance}
For each $n \ge 1$, the conditional variance of the noise is constant across time; that is, there exists a sequence of constants $(\sigma_\eta^{(n)})^2 > 0$ such that, for all $1 \le t \le n$,
\[
\mathbb{E}\left[\big(\eta_t^{(n)}\big)^2 \;\middle|\; \mathcal{F}_{t-1}^{(n)}, A_t^{(n)} \right] = (\sigma_\eta^{(n)})^2.
\]
Moreover, the sequence of variances satisfies
\begin{equation}
\label{eq:limit:sigma_eta}
\lim_{n \to \infty} (\sigma_\eta^{(n)})^2 = \sigma_\eta^2 \in (0, \infty).
\end{equation}
\end{assumption}

We now study the asymptotic distribution of the regularized least-squares estimator based on data collected by LinUCB. Recall that $\mathcal{P}_D$ denotes the set of D-optimal designs. Although the D-optimal design solution itself does not have to be unique, the corresponding information matrix is uniquely determined across the entire optimal set due to the strict concavity of the D-optimality objective $\Phi(\pi)$ over the design space \citep{Fedorov1972}. We may therefore unambiguously define the unique D-optimal information matrix as
\[
M^\star : = \int_{\mathcal{A}} a a^\top \mathrm{d}\pi^\star(a) \quad \text{for any } \pi^\star \in \mathcal{P}_D.
\]

The next theorem establishes the asymptotic normality of the regularized least-squares estimator under LinUCB in the small-gap regime. 

\begin{theorem}[Asymptotic normality of LinUCB in the small-gap regime]
\label{thm:small-gap-inference}
Suppose that Assumptions~\ref{ass:subgua_noise}--\ref{ass:conditional-noise-variance} hold and that $\operatorname{span}(\mathcal{A}) = \mathbb{R}^d$. Then, as $n \to \infty$, the regularized least-squares estimator generated by
LinUCB satisfies
\[
\sqrt{n}\left(\widehat{\theta}_n^{(n)} - \theta_{\star}^{(n)}\right) \Rightarrow \mathcal{N}\left(0, \sigma_\eta^2 (M^\star)^{-1}\right).
\]
\end{theorem}

Theorem~\ref{thm:small-gap-inference} establishes that the regularized least-squares estimator satisfies a standard $\sqrt n$-asymptotic normality result, even though the sampling process is adaptive and not independent. This is consistent with the intuition suggested by the optimal-design limit: the sampling distribution induced by LinUCB converges to an optimal design for learning the unknown reward parameter.   The limiting covariance is determined by the inverse of the D-optimal information matrix $M^\star$, scaled by the limiting noise variance $\sigma_\eta^2$. Hence, $M^\star$ acts as the macroscopic information matrix governing the limiting precision of the estimator.

Theorem~\ref{thm:small-gap-inference} provides an inferential picture that differs from the fixed-gap analysis of LinUCB for the unit-ball action set studied by \citet{fan2025statistical}. Our result exhibits the standard $\sqrt{n}$ convergence rate, whereas in the fixed-gap regime the convergence rate may deteriorate to $n^{1/4}$  in directions orthogonal to the true parameter; see Theorem 1 of~\citet{fan2025statistical}. This difference arises because, in the small-gap regime, LinUCB cannot reliably distinguish the optimal actions from suboptimal alternatives. As a result, the algorithm continues to explore all relevant directions, allowing the design matrix to accumulate information throughout the action space. In contrast, the $n^{1/4}$ rate of  \citet{fan2025statistical}  is anisotropic: it occurs only in directions orthogonal to the true parameter because LinUCB eventually commits to nearly optimal actions and ceases exploring orthogonal directions. In the small-gap regime considered here, the $n^{-1/2}$  scaling of the gaps prevents such commitment, so all directions are explored sufficiently often, restoring the standard 
$\sqrt{n}$ rate in every direction.

Theorem~\ref{thm:small-gap-inference}, together with the fact that the small-gap regime captures hard instances for regret minimization, reveals a contrast between the small-gap and fixed-gap regimes under LinUCB: although the small-gap regime is statistically more challenging for regret minimization, it induces a more information-efficient sampling allocation for parameter learning.

\subsection{Proof of Theorem~\ref{thm:small-gap-inference}}

In this subsection, we prove Theorem~\ref{thm:small-gap-inference}.
First, we leverage the optimal-design limit to demonstrate that the normalized design matrix generated under LinUCB converges in probability to the D-optimal information matrix $M^\star$.  Second, we apply a martingale central limit theorem to handle the adaptive noise term. Combining these two ingredients yields the asymptotic normality of the
estimator.

\begin{proof}[Proof of Theorem~\ref{thm:small-gap-inference}]
First, we show that the normalized design matrix converges. By the definition of $\pi_n^{(n)}$, we have
\[
\frac1n\sum_{t=1}^n A_t^{(n)}(A_t^{(n)})^\top
=
\int_{\mathcal A} aa^\top  d\pi_n^{(n)}(a).
\]


Because $\mathcal A$ is compact, the map $a\mapsto aa^\top$ is bounded
and Lipschitz. Therefore, the $W_1$-convergence to the set
$\mathcal P_D$ implies
\[
\left\|
\int aa^\top  d\pi_n^{(n)}(a)-M^\star
\right\|
\xrightarrow{\mathbb P}0.
\]
Thus
\begin{align}\label{limit:MD}
\frac1n\sum_{t=1}^n A_t^{(n)}(A_t^{(n)})^\top
\xrightarrow{\mathbb P}M^\star.
\end{align}
We also have
\begin{align}\label{limit:Vn}
\frac1n V_n^{(n)}=\frac{\lambda}{n}I_d+
\frac1n\sum_{t=1}^n A_t^{(n)}(A_t^{(n)})^\top
\xrightarrow{\mathbb P}M^\star.
\end{align}
Moreover, since $\operatorname{span}(\mathcal A)=\mathbb R^d$, the
D-optimal information matrix $M^\star$ is positive definite. 

By~\eqref{eq:least_square}, we have
\[
\widehat \theta_n^{(n)}-\theta_{\star}^{(n)}
=
\bigl(V_n^{(n)}\bigr)^{-1}
\sum_{t=1}^n A_t^{(n)}\eta_t^{(n)}
-
\lambda\bigl(V_n^{(n)}\bigr)^{-1}\theta_{\star}^{(n)}.
\]
Multiplying by $\sqrt n$, we have
\[
\sqrt n(\widehat \theta_n^{(n)}-\theta_{\star}^{(n)})
=
\left(\frac1nV_n^{(n)}\right)^{-1}
\frac1{\sqrt n}\sum_{t=1}^n A_t^{(n)}\eta_t^{(n)}
-
\frac{\lambda}{\sqrt n}
\left(\frac1nV_n^{(n)}\right)^{-1}\theta_{\star}^{(n)}.
\]
Since $\{\theta_{\star}^{(n)}\}_{n\ge1}$ is uniformly bounded and
$(n^{-1}V_n^{(n)})^{-1}=O_{\mathbb P}(1)$ by~\eqref{limit:Vn}, the second
term is $o_{\mathbb P}(1)$.

It remains to prove a martingale central limit theorem for
\[
\frac1{\sqrt n}\sum_{t=1}^n A_t^{(n)}\eta_t^{(n)}.
\]
For any fixed $u\in\mathbb R^d$, define
\[
Z_{n,t}=\frac1{\sqrt n}u^\top A_t^{(n)}\eta_t^{(n)}.
\]
Since 
the action $A_t^{(n)}$ is 
$\mathcal F_{t-1}^{(n)}$-measurable, and 
\[
\mathbb E \left[\eta_t^{(n)}\mid  \mathcal F_{t-1}^{(n)},A_t^{(n)}\right]=0,
\]
$\{Z_{n,t}\}_{t=1}^n$ is a martingale difference array. 
  And
\[
\sum_{t=1}^n\mathbb E\left[Z_{n,t}^2
\mid \mathcal F_{t-1}^{(n)},A_t^{(n)}\right]
=(\sigma_\eta^{(n)})^2u^\top\left(\frac1n\sum_{t=1}^n A_t^{(n)}(A_t^{(n)})^\top\right)u.
\]
By the design-matrix convergence~\eqref{limit:MD} and the variance convergence~\eqref{eq:limit:sigma_eta},
\begin{align}\label{limit: quadratic_variation}
\sum_{t=1}^n \mathbb E\left[Z_{n,t}^2
\mid \mathcal F_{t-1}^{(n)},A_t^{(n)}\right]\xrightarrow{\mathbb P}
\sigma_\eta^2 u^\top M^\star u.
\end{align}
Now we verify the conditional Lindeberg condition. Fix
$\varepsilon>0$. By Assumption~\ref{ass:bounded-action-set},
\[
|u^\top A_t^{(n)}|\le \|u\|L=:B_u.
\]
Therefore
\[
|Z_{n,t}| \le \frac{B_u}{\sqrt n}|\eta_t^{(n)}|.
\]
Thus the event $\{|Z_{n,t}|>\varepsilon\}$ implies
\[
|\eta_t^{(n)}|> \frac{\varepsilon\sqrt n}{B_u}.
\]
By uniform conditional subgaussianity of $\{\eta_t^{(n)}\}_{1\le t\le n}$, for any fixed $\delta>0$, there is
a constant $C_\delta<\infty$ such that
\[
\sup_{n,t}\mathbb E\left[|\eta_t^{(n)}|^{2+\delta}
\mid \mathcal F^{(n)}_{t-1}\right]
\le C_\delta \qquad \text{a.s.}
\]
Using the inequality
\[
x^2\mathbf 1\{|x|>K\} \le \frac{|x|^{2+\delta}}{K^\delta},
\]
with $K=\varepsilon\sqrt n/B_u$, we get
\begin{align}
\sum_{t=1}^n\mathbb E \left[
Z_{n,t}^2\mathbf 1\{|Z_{n,t}|>\varepsilon\}
\mid \mathcal F^{(n)}_{t-1}\right] 
&\le \sum_{t=1}^n \frac{B_u^2}{n}\mathbb E \left[
(\eta_t^{(n)})^2\mathbf 1\left\{
|\eta_t^{(n)}|>\frac{\varepsilon\sqrt n}{B_u}\right\}
\mid \mathcal F^{(n)}_{t-1}\right] \nonumber\\
&\le \sum_{t=1}^n \frac{B_u^2}{n}
\left(\frac{B_u}{\varepsilon\sqrt n}\right)^\delta
\mathbb E \left[|\eta_t^{(n)}|^{2+\delta}
\mid \mathcal F^{(n)}_{t-1}\right] \nonumber\\
&\le \frac{B_u^{2+\delta}C_\delta}{\varepsilon^\delta n^{\delta/2}}
\to 0.\label{con:lindeberg}
\end{align}
Therefore the conditional Lindeberg condition holds. By the martingale central limit theorem for martingale arrays, see
Theorem~3.2 and Corollary~3.1, together with the remark following
Corollary~3.1 in \citep{HallHeyde1980}, the convergence of the predictable quadratic
variation~\eqref{limit: quadratic_variation} and the conditional Lindeberg condition~\eqref{con:lindeberg} imply the central
limit theorem
\[
\frac1{\sqrt n}\sum_{t=1}^n u^\top A_t^{(n)}\eta_t^{(n)}
\Rightarrow
\mathcal N(0,\sigma_\eta^2u^\top M^\star u).
\]
Since $u$ is arbitrary, we have
\[
\frac1{\sqrt n}\sum_{t=1}^n A_t^{(n)}\eta_t^{(n)}
\Rightarrow
\mathcal N(0,\sigma_\eta^2M^\star).
\]
Finally, combining this distribution limit with the matrix convergence in \eqref{limit:Vn} via Slutsky's theorem yields
\[
\sqrt n(\widehat \theta_n^{(n)}-\theta_{\star}^{(n)})
\Rightarrow (M^\star)^{-1}\mathcal N(0,\sigma_\eta^2M^\star)
= \mathcal N(0,\sigma_\eta^2(M^\star)^{-1}).
\]
This completes the proof. 
\end{proof}

\section{Conclusion}
\label{sec:conclusion}

We establish that, in the small-gap regime, the empirical sampling distribution induced by LinUCB converges to the set of D-optimal designs. This result adopts a mean-field perspective: the macroscopic sampling limit is characterized by the maximizers of the D-optimality criterion, which serves as a potential function. It shows that LinUCB gradually forms a sampling distribution that is statistically efficient for learning the unknown reward parameter, thereby connecting regret-driven online learning with information-efficient experimental design. Building on this optimal-design limit, we provide a refined asymptotic analysis of LinUCB’s pseudo-regret in the small-gap regime by characterizing its leading-order constant. Finally, we show that this limit yields standard $\sqrt{n}$-asymptotic normality of the regularized least-squares estimator, thereby justifying statistical inference for the reward parameter based on adaptively collected LinUCB data.

Several directions remain open. First, in settings where the D-optimal design is not unique, it remains important to determine which D-optimal design the empirical sampling distribution converges to in the limit. Addressing this question may require developing fluid or diffusion approximations for the regret or reward processes of LinUCB in the small-gap regime, and could yield a sharper characterization of the leading-order constant for the pseudo-regret in the nonunique case. Second, a natural next step is to extend the analysis beyond the small-gap regime by developing a unified mean-field description of the full phase transition across the small-gap, critical-gap, and fixed-gap regimes, in the spirit of the two-arm bandit analysis of \citet{KalvitZeevi2021}. Such a theory would provide a fuller answer to the challenging problem of sampling allocation in correlated linear bandits, and would clarify how this allocation encodes information beyond regret, such as critical thresholds for valid statistical inference on the reward parameter.

\bibliography{references}
\bibliographystyle{informs2014}

\appendix
\section{Proofs of the Main Technical Results}
\label{app:proofs-main-technical-results}

\subsection{Proof of Proposition~\ref{prop:small-gap-regret-limit} and Subsequent Discussion}
\label{appsub:proofs-coro}

\begin{proof}[Proof of Proposition~\ref{prop:small-gap-regret-limit}]
By~\eqref{theta_scale}, for every $a\in\mathcal A$,
\[
\max_{b\in\mathcal A}\langle b,\theta_\star^{(n)}\rangle
-\langle a,\theta_\star^{(n)}\rangle
=
\frac{1}{\sqrt n}\Delta(a).
\]
Therefore, by the definition of the empirical sampling distribution $\pi_n^{(n)}$,
we have
\[
\frac{\mathcal R_n}{\sqrt n}=
\frac1n\sum_{t=1}^n \Delta(A_t^{(n)})
=\int_{\mathcal A}\Delta(a)\,\pi_n^{(n)}({\rm d}a).
\]
Define the functional $\Phi:\mathcal P(\mathcal A)\to\mathbb R$ by
\[
\Phi(\mu):=
\int_{\mathcal A}\Delta(a)\mu({\rm d}a).
\]
The function $\Delta$ is Lipschitz on $\mathcal A$ with Lipschitz constant
at most $\|\theta_0\|$. Hence, for any probability measures $\mu$ and $\nu$
on $\mathcal A$,
\[
|\Phi(\mu)-\Phi(\nu)|\le\|\theta_0\|\, W_1(\mu,\nu).
\]
It follows that
\[
\operatorname{dist}
\left(
\Phi(\pi_n^{(n)}),
\{\Phi(\pi):\pi\in\mathcal P_D\}
\right)
\le
\|\theta_0\|\,
\operatorname{dist}_{W_1}\left(\pi_n^{(n)},\mathcal P_D\right).
\]
Then by Theorem~\ref{thm:linucb-general-D-optimal},
\[
\operatorname{dist}
\left(
\frac{\mathcal R_n}{\sqrt n},
\left\{
\int_{\mathcal A}\Delta(a)\,\pi({\rm d}a):
\pi\in\mathcal P_D
\right\}
\right)
\xrightarrow{\mathbb P}0.
\]

If $\mathcal P_D=\{\pi^\star\}$, then the limiting set is the singleton, and hence
\[
\frac{\mathcal R_n}{\sqrt n}
\xrightarrow{\mathbb P}
\int_{\mathcal A}\Delta(a)\,\pi^\star({\rm d}a).
\]
This completes the proof.
\end{proof}

\begin{proof}[Proof of the attainability step~\eqref{eq:information matrix} in Corollary~\ref{coro:hypercube}]
We identify the D-optimal design on $\mathcal A$. Let
\[
c_d:=
\begin{cases}
\dfrac{d+1}{4d}, & d \text{ odd},\\[6pt]
\dfrac{d+2}{4(d+1)}, & d \text{ even}.
\end{cases}
\]
The D-optimal information matrix is
\[
M^\star=c_d(I_d+\mathbf 1\mathbf 1^\top).
\]
Indeed, since
\[
(M^\star)^{-1}=\frac1{c_d}
\left(I_d-\frac{\mathbf 1\mathbf 1^\top}{d+1}\right),
\]
 for any $a\in\mathcal A$ with $k=\sum_{i=1}^d a_i$, we have
\[
a^\top (M^\star)^{-1}a=
\frac1{c_d}\left(k-\frac{k^2}{d+1}\right)
=\frac{k(d+1-k)}{c_d(d+1)}.
\]
By the choice of $c_d$,
\[
\max_{a\in\mathcal A} a^\top (M^\star)^{-1}a=d.
\]
Therefore, by the Kiefer--Wolfowitz equivalence theorem, $M^\star$ is
D-optimal. 
This finishes the proof of the attainability step~\eqref{eq:information matrix} in Corollary~\ref{coro:hypercube}.
\end{proof}

\subsection{Proof of the Concentration Bound}
\label{app:proof-concentration-bound}

\begin{proof}[Proof of Lemma~\ref{lem:general-linear-concentration}]
For notational simplicity, we suppress the superscript $^{(n)}$ throughout the
proof. 
For each action $a\in\mathcal A$, define
\[
I_t(a):=\mathbf 1\{A_t=a\},
\qquad
N_a(t):=\sum_{s=1}^t I_s(a),
\qquad
Q_t(a):=\sum_{s=1}^t I_s(a)\eta_s .
\]
Thus $N_a(t)$ is the number of times action $a$ has been selected up to
time $t$, and $Q_t(a)$ is the corresponding accumulated noise.

We first prove a maximal concentration bound for $Q_t(a)$. Fix
$a\in\mathcal A$ and $u\in\mathbb R$. Define
\[
Z_t(a,u):=\exp\left\{uQ_t(a)-\frac{u^2\sigma^2}{2}N_a(t)
\right\},\qquad 0\le t\le n.
\]
We claim that $\{Z_t(a,u)\}_{t=0}^n$ is a nonnegative supermartingale with
respect to $\{\mathcal F_t\}_{t=0}^n$. Indeed, $I_t(a)$ is
$\mathcal H_t$-measurable, where $\mathcal H_t:=\sigma(\mathcal F_{t-1},A_t)$. By Assumption~\ref{ass:subgua_noise},
\[
\mathbb E\left[\exp\left\{uI_t(a)\eta_t-\frac{u^2\sigma^2}{2}I_t(a)\right\}
\mid \mathcal H_t\right]\le 1.
\]
Therefore, using the tower property,
\[
\mathbb E[Z_t(a,u)\mid \mathcal F_{t-1}]
\le Z_{t-1}(a,u).
\]
This proves the supermartingale claim.

Since for every $M>0$, $x>0$, and $u>0$,
\[
\left\{\exists t\le n:\ N_a(t)\le M,\ Q_t(a)>x\right\}
\subseteq\left\{\max_{1\le t\le n} Z_t(a,u)
>\exp\left\{ux-\frac{u^2\sigma^2}{2}M\right\}\right\}.
\]
By Ville's inequality, 
\[
\mathbb P\left(\exists t\le n:\ N_a(t)\le M,\ Q_t(a)>x\right)
\le \exp\left\{-ux+\frac{u^2\sigma^2}{2}M\right\}.
\]
Optimizing over $u>0$ gives
\[
\mathbb P\left(\exists t\le n:\ N_a(t)\le M,\ Q_t(a)>x\right)
\le \exp\left(-\frac{x^2}{2\sigma^2 M}\right).
\]
Applying the same argument to $-Q_t(a)$, we obtain
\[
\mathbb P\left(
\exists t\le n:\ N_a(t)\le M,\ |Q_t(a)|>x\right)
\le 2\exp\left(-\frac{x^2}{2\sigma^2 M}\right).
\]

Now set $J_n:=\lceil \log_2 n\rceil$ and $\xi=2\sigma+1$. For each fixed $a\in\mathcal A$,
define
\[
E_{a,n}:=\left\{\exists t\le n:
|Q_t(a)|>\xi\sqrt{N_a(t)\log\log n}\right\}.
\]
If $N_a(t)=0$, then $Q_t(a)=0$, so the event can only occur when
$N_a(t)\ge 1$. Decomposing according to dyadic ranges of $N_a(t)$, we have
\[
E_{a,n}\subseteq
\bigcup_{r=1}^{J_n}\left\{\exists t\le n:N_a(t)\le 2^r,\
|Q_t(a)|>\xi\sqrt{2^{r-1}\log\log n}\right\}.
\]
Hence,
\begin{align*}
\mathbb P(E_{a,n})\le\sum_{r=1}^{J_n}
2\exp\left(-\frac{\xi^2 2^{r-1}\log\log n}{2\sigma^2 2^r}\right)=
2J_n(\log n)^{-\xi^2/(4\sigma^2)}.
\end{align*}
Since $\mathcal A$ is finite and $\xi>2\sigma$, a union bound over
$\mathcal A$ gives
\[
\mathbb P(\mathcal E_n)\to 1,
\]
where
\[
\mathcal E_n:=\left\{\forall a\in\mathcal A,\ \forall t\in\{1,\ldots,n\},\ |Q_t(a)|\le \xi\sqrt{N_a(t)\log\log n}
\right\}.
\]

We now work on the event $\mathcal E_n$. From the definition of the
regularized least-squares estimator,
\[
\widehat\theta_t=
V_t^{-1}\sum_{s=1}^t A_sX_s=
V_t^{-1}\sum_{s=1}^t A_s(A_s^\top\theta_\star+\eta_s).
\]
We have
\[
\widehat\theta_t-\theta_\star
=-\lambda V_t^{-1}\theta_\star
+V_t^{-1}\sum_{s=1}^t A_s\eta_s.
\]
Therefore, for any $a\in\mathcal A$,
\[
\left\langle
\widehat\theta_t-\theta_\star,a
\right\rangle=-\lambda a^\top V_t^{-1}\theta_\star
+a^\top V_t^{-1}\sum_{s=1}^t A_s\eta_s.
\]

We first bound the stochastic term. Since $\mathcal A$ is finite,
\[
\sum_{s=1}^t A_s\eta_s=\sum_{b\in\mathcal A} b Q_t(b).
\]
Thus, on $\mathcal E_n$,
\begin{align*}
\left|a^\top V_t^{-1}\sum_{s=1}^t A_s\eta_s\right|
&\le\sum_{b\in\mathcal A}|a^\top V_t^{-1}b| |Q_t(b)| \\
&\le\xi\sqrt{\log\log n}\sum_{b\in\mathcal A}
|a^\top V_t^{-1}b|\sqrt{N_b(t)}.
\end{align*}
By Cauchy's inequality,
\[
\sum_{b\in\mathcal A}|a^\top V_t^{-1}b|\sqrt{N_b(t)}
\le\sqrt |\mathcal A|\left(\sum_{b\in\mathcal A}
N_b(t)(a^\top V_t^{-1}b)^2\right)^{1/2}.
\]
Moreover,
\begin{align*}
\sum_{b\in\mathcal A}N_b(t)(a^\top V_t^{-1}b)^2&=
a^\top V_t^{-1}\left(\sum_{b\in\mathcal A}N_b(t)bb^\top
\right)V_t^{-1}a \\
&=a^\top V_t^{-1}(V_t-\lambda I_d)V_t^{-1}a \\
&=a^\top V_t^{-1}a-\lambda a^\top V_t^{-2}a \\
&\le a^\top V_t^{-1}a.
\end{align*}
Therefore,
\[
\left|a^\top V_t^{-1}\sum_{s=1}^t A_s\eta_s\right|
\le \xi\sqrt{|\mathcal A|\log\log n}  \|a\|_{V_t^{-1}}.
\]

It remains to bound the regularization bias. By Cauchy's inequality in the
$V_t^{-1}$-norm,
\[
\lambda |a^\top V_t^{-1}\theta_\star|
\le\lambda\|a\|_{V_t^{-1}}\|\theta_\star\|_{V_t^{-1}}.
\]
Since $V_t\succeq \lambda I_d$, we have
$V_t^{-1}\preceq \lambda^{-1}I_d$. Hence, by
Assumption~\ref{ass:bounded-parameters},
\[
\|\theta_\star\|_{V_t^{-1}}
\le\lambda^{-1/2}\|\theta_\star\|_2
\le\lambda^{-1/2}C_{\theta_\star}.
\]
Thus
\[
\lambda |a^\top V_t^{-1}\theta_\star|
\le\sqrt{\lambda}C_{\theta_\star}\|a\|_{V_t^{-1}}.
\]

Combining the stochastic term and the bias term, on $\mathcal E_n$, for all
$t\in\{1,\ldots,n\}$ and all $a\in\mathcal A$,
\[
\left|\left\langle\widehat\theta_t-\theta_\star,a
\right\rangle\right|
\le\left(\sqrt{\lambda}C_{\theta_\star}
+\xi\sqrt{|\mathcal A|\log\log n}\right)\|a\|_{V_t^{-1}}.
\]
Since $\mathbb P(\mathcal E_n)\to 1$, the desired result follows.
\end{proof}

\subsection{Proof of Proposition~\ref{prop:linucb-approx-fw}}
\label{app:variance-maximization}

 We now prove Proposition~\ref{prop:linucb-approx-fw}, which shows that, after a negligible initial period, the LinUCB selection rule approximately maximizes the prediction variance.

\begin{proof}[Proof of Proposition~\ref{prop:linucb-approx-fw}]
For $a\in\mathcal A$, define
\[
q_t^{(n)}(a):=\left\|a\right\|_{\left(V_{t-1}^{(n)}\right)^{-1}},
\]
and then
\[
I_t^{(n)}(a):=a^\top \widehat\theta_{t-1}^{(n)}+\sqrt{\beta_t}  q_t^{(n)}(a)
\]
be the LinUCB index.

Fix $r_n\le t\le n$. Since $\mathcal A$ is compact under Assumption~
\ref{ass:bounded-action-set} and the map
$a\mapsto q_t^{(n)}(a)$ is continuous, there exists
\[
a_t^\star\in\operatorname*{arg  max}_{a\in\mathcal A}q_t^{(n)}(a).
\]
Write
\[
q_t^\star:=q_t^{(n)}(a_t^\star)=\max_{a\in\mathcal A}q_t^{(n)}(a).
\]

Let $A_t^{(n)}$ be the action selected by LinUCB at round $t$, and write
$q_t^{\mathrm{sel}}:=q_t^{(n)}(A_t^{(n)})$.
We claim that there exists a deterministic sequence $\gamma_n\to0$, to be specified below, such that, for all sufficiently large $n$ and all $r_n\le t\le n$,
\[
q_t^{\mathrm{sel}}\ge\frac{1}{1+\gamma_n}q_t^\star .
\]
Suppose, to the contrary, that for some sufficiently large $n$ this assertion fails. Then there exists a time $t\in\{r_n,\ldots,n\}$ such that
\[
q_{t}^{\mathrm{sel}}<\frac{1}{1+\gamma_{n}}q_{t}^\star .
\]
Consequently,
\begin{align}\label{ieq:contrary}
q_t^\star-q_t^{\mathrm{sel}}>\frac{\gamma_n}{1+\gamma_n}q_t^\star,
\qquad 
q_t^\star+q_t^{\mathrm{sel}}<\left(1+\frac{1}{1+\gamma_n}\right)q_t^\star
\le 2q_t^\star .
\end{align}
Indeed, on $\mathcal E_n$, by~Lemma~\ref{lem:general-linear-concentration},
\[
\begin{aligned}
I_t^{(n)}(a_t^\star)-I_t^{(n)}(A_t^{(n)}) 
&=\left\langle a_t^\star-A_t^{(n)},\widehat\theta_{t-1}^{(n)}
\right\rangle+\sqrt{\beta_t} \left(q_t^\star-q_t^{\mathrm{sel}}\right) \\
&=\left\langle a_t^\star-A_t^{(n)},\theta_\star^{(n)}
\right\rangle+\left\langle
\widehat\theta_{t-1}^{(n)}-\theta_\star^{(n)},a_t^\star
\right\rangle-\left\langle\widehat\theta_{t-1}^{(n)}-\theta_\star^{(n)},A_t^{(n)}\right\rangle\\
&\qquad  +\sqrt{\beta_t}\left(q_t^\star-q_t^{\mathrm{sel}}\right) \\
&\ge\left\langle a_t^\star-A_t^{(n)},\theta_\star^{(n)}
\right\rangle
-C_n q_t^\star-C_n q_t^{\mathrm{sel}}
+\sqrt{\beta_t}\left(q_t^\star-q_t^{\mathrm{sel}}\right).
\end{aligned}
\]
We first lower bound $q_t^\star$. Let
\[
L_{\mathcal A}:=\max_{a\in\mathcal A}\|a\|_2 .
\]
Since $\operatorname{span}(\mathcal A)=\mathbb R^d$, we have
$L_{\mathcal A}>0$. Moreover, by Assumption~\ref{ass:bounded-action-set},
\[
V_{t-1}^{(n)}=\lambda I_d+\sum_{s=1}^{t-1}A_s^{(n)}\bigl(A_s^{(n)}\bigr)^\top
\preceq(\lambda+nL^2)I_d.
\]
Hence
\[
\left(V_{t-1}^{(n)}\right)^{-1}\succeq\frac{1}{\lambda+nL^2}I_d.
\]
Therefore,
\[
q_t^\star=\max_{a\in\mathcal A}
\left\|a\right\|_{\left(V_{t-1}^{(n)}\right)^{-1}}
\ge\frac{L_{\mathcal A}}{\sqrt{\lambda+nL^2}}.
\]

Thus, by Assumption~\ref{asm small gap}, with
\[
R_{\mathcal A}:=
\frac{C_{\mathcal{A}}}{L_{\mathcal A}}\sqrt{\lambda+L^2},
\]
we have, for all $n\ge 1$,
\begin{align}\label{ieq:mean}
\left\langle a_t^\star-A_t^{(n)},\theta_\star^{(n)}
\right\rangle
\ge
-\frac{C_{\mathcal{A}}}{\sqrt n}
\ge
- R_{\mathcal A} q_t^\star .
\end{align}
(This seems to be the only place where we actually use Assumption~\ref{asm small gap}; anything generalizing this step can also generalize the small gap condition.)

Therefore, by~\eqref{ieq:contrary} and~\eqref{ieq:mean}, we have
\[
\begin{aligned}
I_t^{(n)}(a_t^\star)-I_t^{(n)}(A_t^{(n)}) 
&\ge-R_{\mathcal A}q_t^\star-2C_nq_t^\star
+\frac{\gamma_n}{1+\gamma_n}\sqrt{\beta_t}  q_t^\star \\
&=\left[\frac{\gamma_n}{1+\gamma_n}\sqrt{\beta_t}-
2C_n-R_{\mathcal A}\right]q_t^\star .
\end{aligned}
\]

Now choose
\[
\gamma_n:=\frac{6C_n+4R_{\mathcal A}}{\sqrt{\beta_{r_n}}}.
\]
Since $\sqrt{\beta_{r_n}}\ge\sqrt{\log r_n}=\log\log n$,
and $C_n=\Theta(\sqrt{\log\log n})$, we have $\gamma_n\to0$.
Thus, for all sufficiently large $n$, $\gamma_n\in(0,1]$. Moreover, for $r_n\le t\le n$, $\sqrt{\beta_t}\ge \sqrt{\beta_{r_n}}$.
Hence, for all sufficiently large $n$,
\[
\frac{\gamma_n}{1+\gamma_n}\sqrt{\beta_t}\ge\frac{\gamma_n}{2}\sqrt{\beta_{r_n}}=3C_n+2R_{\mathcal A}.
\]
Consequently,
\[
I_t^{(n)}(a_t^\star)-I_t^{(n)}(A_t^{(n)})\ge(C_n+R_{\mathcal A})q_t^\star>0,
\]
which contradicts the fact that LinUCB selects an index maximizer. Hence,
\[
q_t^{(n)}(A_t^{(n)})
\ge
\frac{1}{1+\gamma_n}
\max_{a\in\mathcal A}q_t^{(n)}(a).
\]
Finally, define $\varepsilon_n:=1-\frac{1}{(1+\gamma_n)^2}$.
Since $\gamma_n\to0$, we have $\varepsilon_n\to0$. This
completes the proof.
\end{proof}

\subsection{Proof of Proposition~\ref{prop:fw-triangular-array-robustness}}
\label{app:fw-convergence}

We next prove Proposition~\ref{prop:fw-triangular-array-robustness}, the approximate Fedorov--Wynn convergence result. The proof follows a determinant-growth argument: approximate variance maximization drives the normalized information matrix toward the optimal determinant value, and compactness then converts this into convergence of the empirical designs.

\begin{proof}[Proof of Proposition~\ref{prop:fw-triangular-array-robustness}]
Since $\mathcal A$ is compact under Assumption~\ref{ass:bounded-action-set}, Prokhorov's theorem implies that $\mathcal P(\mathcal A)$ is compact in the  weak topology. Moreover, on a
compact metric space, weak convergence and $W_1$-convergence are equivalent. Hence $\mathcal P(\mathcal A)$ is also compact with respect to $W_1$. The map
\[
\pi\mapsto M(\pi)=\int_{\mathcal A}aa^\top  \pi({\rm d} a)
\]
is continuous, because $a\mapsto aa^\top$ is continuous and bounded on
$\mathcal A$. Hence $\mathcal P_D$ is nonempty. By
$\operatorname{span}(\mathcal A)=\mathbb R^d$, there exists a design with a
positive definite information matrix, so
\[
D^\star:=\max_{\pi\in\mathcal P(\mathcal A)}\det M(\pi)>0.
\]

Fix $\pi^\star\in\mathcal P_D$, and write
\[
M^\star:=M(\pi^\star), \qquad
c^\star:=(\det M^\star)^{1/d}=(D^\star)^{1/d}>0.
\]
For readability, we suppress the superscript $(n)$ in the following. Define
\[
h_s:=Z_s^\top B_{s-1}^{-1}Z_s,\qquad y_s:=\det(B_s)^{1/d}.
\]
By the matrix determinant lemma,
\[
\det(B_s)=\det(B_{s-1})(1+h_s),
\]
and hence
\[
y_s=y_{s-1}(1+h_s)^{1/d}.
\]

For every positive definite matrix $B_{s-1}$,
\[
\begin{aligned}
\max_{a\in\mathcal A}a^\top B_{s-1}^{-1}a\ge
\int_{\mathcal A}a^\top B_{s-1}^{-1}a  \pi^\star({\rm d} a) =
\operatorname{tr} \left(B_{s-1}^{-1}M^\star\right) \ge
d\left(\frac{\det M^\star}{\det B_{s-1}}\right)^{1/d} =
\frac{d c^\star}{y_{s-1}},
\end{aligned}
\]
where the inequality follows from the arithmetic--geometric mean inequality
applied to the eigenvalues of
$B_{s-1}^{-1/2}M^\star B_{s-1}^{-1/2}$. Therefore, by
\eqref{con_appro_large},
\[
h_s\ge (1-\delta_n)\frac{d c^\star}{y_{s-1}}.
\]
Set $\alpha_n:=(1-\delta_n)c^\star$.
For all sufficiently large $n$, $\alpha_n\in[c^\star/2,c^\star]$. Hence
\[
y_s-y_{s-1}\ge y_{s-1} \left[\left(1+\frac{d\alpha_n}{y_{s-1}}\right)^{1/d}-1\right].
\]
Since $B_s\succeq B_0\succeq \lambda I_d$, we have $y_s\ge\lambda$ for
all $s$. Then for $\alpha\in[c^\star/2,c^\star]$ and $y\ge\lambda$, define
\[
G_\alpha(y):=y\left[\left(1+\frac{d\alpha}{y}\right)^{1/d}-1\right].
\]
 Moreover, there exists $v>0$ such that
\begin{align}\label{ieq_G_alpha1}
G_\alpha(y)\ge v, \qquad \alpha\in[c^\star/2,c^\star],\ y\ge\lambda.
\end{align}
Indeed, $G_\alpha(y)>0$, and
\[
\lim_{y\to\infty}G_\alpha(y)=\alpha\ge c^\star/2,
\]
so the claim follows by continuity.

Also, there exists $C<\infty$ such that
\begin{align}\label{ieq_G_alpha2}
G_\alpha(y)\ge \alpha-\frac{C}{y}, \qquad \alpha\in[c^\star/2,c^\star],\ y\ge\lambda.
\end{align}
This follows from the uniform Taylor expansion of $(1+x)^{1/d}$ on the
compact interval $x\in[0,dc^\star/\lambda]$.

Applying the bound~\eqref{ieq_G_alpha1} with $\alpha=\alpha_n$ gives
\[
y_s-y_{s-1}\ge G_{\alpha_n}(y_{s-1})\ge v,
\]
and hence
\[
y_{s-1}\ge y_0+v(s-1)\ge \lambda+v(s-1).
\]
Using the bound~\eqref{ieq_G_alpha2}, we also have
\[
y_s-y_{s-1}\ge \alpha_n-\frac{C}{y_{s-1}}.
\]
Summing over $s=1,\ldots,m_n$, we obtain
\[
y_{m_n}\ge y_0+m_n\alpha_n-C\sum_{s=1}^{m_n}\frac1{y_{s-1}} 
\ge m_n(1-\delta_n)c^\star -C\sum_{s=1}^{m_n}\frac1{\lambda+v(s-1)}.
\]
Since
\[
\sum_{s=1}^{m_n}\frac1{\lambda+v(s-1)}=\Theta(\log m_n),
\]
we have
\[
\liminf_{n\to\infty}\frac{y_{m_n}}{m_n}\ge c^\star.
\]
Equivalently,
\[
\liminf_{n\to\infty}\det\left(\frac{B_{m_n}^{(n)}}{m_n}\right)\ge D^\star.
\]

On the other hand,
\[
\frac{B_{m_n}^{(n)}}{m_n}=M \left(\bar\pi_{m_n}^{(n)}\right)
+\frac{B_0^{(n)}}{m_n}.
\]
Since $B_0^{(n)}\succ0$ and
\[
\frac{\operatorname{tr}(B_0^{(n)})}{m_n}\to0,
\]
we have
\[
\left\|\frac{B_0^{(n)}}{m_n}\right\|_{\mathrm{op}}\le
\frac{\operatorname{tr}(B_0^{(n)})}{m_n}\to0.
\]

Now take an arbitrary subsequence. By compactness of $\mathcal P(\mathcal A)$,
there exists a further subsequence, still denoted by $n$, such that
\[
\bar\pi_{m_n}^{(n)}\to\pi\qquad\text{in }W_1
\]
for some $\pi\in\mathcal P(\mathcal A)$. Thus,
\[
M \left(\bar\pi_{m_n}^{(n)}\right)\to M(\pi).
\]
Therefore,
\[
\frac{B_{m_n}^{(n)}}{m_n} \to M(\pi),
\]
and hence
\[
\limsup_{n\to\infty}
\det\left(\frac{B_{m_n}^{(n)}}{m_n}\right)
\le \max_{\pi\in\mathcal P(\mathcal A)}\det M(\pi)=
D^\star.
\]
Together with the lower bound, this implies
\[
\det\left(\frac{B_{m_n}^{(n)}}{m_n}\right)\to D^\star.
\]
Finally, suppose for contradiction that
\[
\operatorname{dist}_{W_1} \bigl(\bar\pi_{m_n}^{(n)},\mathcal P_D\bigr)
\nrightarrow 0.
\]
Then there exist $\varepsilon>0$ and a subsequence, still denoted by $n$, such that
\[
\operatorname{dist}_{W_1}\bigl(\bar\pi_{m_n}^{(n)},\mathcal P_D\bigr)
\ge \varepsilon.
\]
By compactness, along a further subsequence (still denoted by $n$),
\[
\bar\pi_{m_n}^{(n)}\to \pi \qquad\text{in }W_1.
\]
The determinant convergence above gives
\[
\det M(\pi)=D^\star,
\]
so $\pi\in\mathcal P_D$. Therefore,
\[
\operatorname{dist}_{W_1} \bigl(\bar\pi_{m_n}^{(n)},\mathcal P_D\bigr)
\le W_1\bigl(\bar\pi_{m_n}^{(n)},\pi\bigr)\to0,
\]
which is a contradiction. Hence
\[
\operatorname{dist}_{W_1}\bigl(\bar\pi_{m_n}^{(n)},\mathcal P_D\bigr)\to0.
\]
This completes the proof.
\end{proof}

\subsection{Proof of Lemma~\ref{lem:late-stage-approx-fw}}
\label{app:late-stage-approx-fw}

We next prove Lemma~\ref{lem:late-stage-approx-fw}, the late-stage version of the approximate Fedorov--Wynn
convergence result. We remove the first $r_n-1$ actions and regard the remaining actions as a new triangular-array sequence, with the removed actions incorporated into the initial information matrix. Since $r_n=o(n)$, this initial matrix has asymptotically negligible normalized trace, and the approximate Fedorov--Wynn convergence result is applicable. It remains to show that removing the first $r_n-1$ actions does not affect the limiting
empirical distribution.

\begin{proof}[Proof of Lemma~\ref{lem:late-stage-approx-fw}]
Let
\[
m_n:=n-r_n+1.
\]
Since $r_n/n\to0$, we have $m_n\to\infty$. Define the late-stage sequence
\[
Z_s^{(n)}:=A_{r_n+s-1}^{(n)}, \qquad s=1,\ldots,m_n,
\]
and the shifted matrices
\[
\widetilde B_s^{(n)}:=B_{r_n+s-1}^{(n)}, \qquad s=0,\ldots,m_n.
\]
Equivalently,
\[
\widetilde B_0^{(n)}=B_{r_n-1}^{(n)}=\lambda I_d+\sum_{u=1}^{r_n-1}
A_u^{(n)}\bigl(A_u^{(n)}\bigr)^\top,
\]
and
\[
\widetilde B_s^{(n)}=\widetilde B_0^{(n)}+\sum_{u=1}^s
Z_u^{(n)}\bigl(Z_u^{(n)}\bigr)^\top.
\]
By condition~\eqref{con:approximate-large-infinite}, for all sufficiently
large $n$ and all $s=1,\ldots,m_n$,
\[
\bigl(Z_s^{(n)}\bigr)^\top\bigl(\widetilde B_{s-1}^{(n)}\bigr)^{-1}
Z_s^{(n)}\ge(1-\varepsilon_n)\max_{a\in\mathcal A}a^\top
\bigl(\widetilde B_{s-1}^{(n)}\bigr)^{-1}a.
\]
By Assumption~\ref{ass:bounded-action-set},
\[
\operatorname{tr}\bigl(\widetilde B_0^{(n)}\bigr)\le \lambda d+(r_n-1)L^2.
\]
Since $r_n/n\to0$ and $m_n=n-r_n+1$, we have
\[
\frac{\operatorname{tr}\bigl(\widetilde B_0^{(n)}\bigr)}{m_n}\le
\frac{\lambda d+(r_n-1)L^2}{n-r_n+1}\to0.
\]
Thus Proposition~\ref{prop:fw-triangular-array-robustness} applies to the late-stage sequence. Hence, with
\[
\bar\pi_{m_n}^{(n)}:=\frac1{m_n}\sum_{s=1}^{m_n}\delta_{Z_s^{(n)}},
\]
we obtain
\[
\operatorname{dist}_{W_1}\bigl(\bar\pi_{m_n}^{(n)},\mathcal P_D\bigr)\to0.
\]

It remains to compare the late-stage empirical distribution
$\bar\pi_{m_n}^{(n)}$ with the full empirical distribution
$\pi_n^{(n)}$. For all sufficiently large $n$, $r_n\ge2$. Define the
early-stage empirical distribution
\[
\pi_{r_n-1}^{(n)}:=\frac1{r_n-1}\sum_{u=1}^{r_n-1}\delta_{A_u^{(n)}}.
\]
Then
\[
\pi_n^{(n)}=\frac{r_n-1}{n}\pi_{r_n-1}^{(n)}+\frac{m_n}{n}\bar\pi_{m_n}^{(n)}.
\]
Since
\[
\sup_{a,b\in\mathcal A}\|a-b\|_2\leq 2L,
\]
then the $W_1$ distance between any two
probability measures supported on $\mathcal A$ is at most $2L$.
Thus, we have
\[
\begin{aligned}
W_1 \left(\pi_n^{(n)},\bar\pi_{m_n}^{(n)}\right)&=
W_1 \left(\frac{r_n-1}{n}\pi_{r_n-1}^{(n)}+\frac{m_n}{n}\bar\pi_{m_n}^{(n)}, \bar\pi_{m_n}^{(n)}\right) \\
&= \frac{r_n-1}{n} W_1 \left(\pi_{r_n-1}^{(n)},\bar\pi_{m_n}^{(n)}\right) \\
&\le \frac{r_n-1}{n}2L \to0.
\end{aligned}
\]
Therefore,
\[
\operatorname{dist}_{W_1}\bigl(\pi_n^{(n)},\mathcal P_D\bigr)
\le W_1 \left(\pi_n^{(n)},\bar\pi_{m_n}^{(n)}\right)
+\operatorname{dist}_{W_1}\bigl(\bar\pi_{m_n}^{(n)},\mathcal P_D\bigr)\to0.
\]
This completes the proof.
\end{proof}

\end{document}